\documentclass{article}

\usepackage[preprint]{neurips_2025}

\usepackage{amsmath}
\usepackage{graphicx}
\usepackage{multirow}

\usepackage{colortbl}

\usepackage[utf8]{inputenc} 
\usepackage[T1]{fontenc}    
\usepackage{hyperref}       
\usepackage{url}            
\usepackage{booktabs}       
\usepackage{amsfonts}       
\usepackage{nicefrac}       
\usepackage{microtype}      
\usepackage{xcolor}         
\usepackage{subcaption}
\usepackage{tcolorbox}
\usepackage[percent]{overpic}

\usepackage{amsmath,amsfonts,bm}

\def\eqref#1{equation~\ref{#1}}
\def\Eqref#1{Equation~\ref{#1}}

\def\1{\bm{1}}

\DeclareMathAlphabet{\mathsfit}{\encodingdefault}{\sfdefault}{m}{sl}
\SetMathAlphabet{\mathsfit}{bold}{\encodingdefault}{\sfdefault}{bx}{n}

\newcommand{\E}{\mathbb{E}}
\renewcommand{\P}{\mathbb{P}}

\newcommand{\R}{\mathbb{R}}

\newcommand{\Cov}{\mathrm{Cov}}

\newcommand{\calN}{\mathcal{N}}

\newcommand{\calK}{\mathcal{K}}

\providecommand{\bydef}{\overset{\text{def}}{=}}

\usepackage[utf8]{inputenc} 
\usepackage[T1]{fontenc}    
\usepackage{hyperref}       
\usepackage{url}            
\usepackage{booktabs}       
\usepackage{amsfonts}       
\usepackage{nicefrac}       
\usepackage{microtype}      
\usepackage{xcolor}         

\usepackage{mathtools}   

\usepackage{amsthm}
\newtheorem{theorem}{Theorem}

\newtheorem{lemma}{Lemma}

\title{Bradley–Terry and Multi‑Objective Reward Modeling Are Complementary}

\author{%
Zhiwei Zhang\textsuperscript{1}\thanks{Work done during internship at Amazon.}, Hui Liu\textsuperscript{2}, Xiaomin Li\textsuperscript{3}, Zhenwei Dai\textsuperscript{2}, Jingying Zeng\textsuperscript{2}, Fali Wang\textsuperscript{1}, Minhua Lin\textsuperscript{1}, \\
\textbf{Ramraj Chandradevan}\textsuperscript{2}, \textbf{Zhen Li}\textsuperscript{2}, \textbf{Chen Luo}\textsuperscript{2}, \textbf{Xianfeng Tang}\textsuperscript{2}, \textbf{Qi He}\textsuperscript{2}, \textbf{Suhang Wang}\textsuperscript{1}\\
\centerline{\textsuperscript{1}The Pennsylvania State University \quad
\textsuperscript{2}Amazon \quad
\textsuperscript{3}Havard University}
}

\begin{document}

\maketitle

\begin{abstract}
Reward models trained on human preference data have demonstrated strong effectiveness in aligning Large Language Models (LLMs) with human intent under the framework of Reinforcement Learning from Human Feedback (RLHF). However, RLHF remains vulnerable to reward hacking, where the policy exploits imperfections in the reward function rather than genuinely learning the intended behavior. 
Although significant efforts have been made to mitigate reward hacking, they predominantly focus on and evaluate in-distribution scenarios, where the training and testing data for the reward model share the same distribution.
In this paper, we empirically show that state-of-the-art methods struggle in more challenging out-of-distribution (OOD) settings. We further demonstrate that incorporating fine-grained multi-attribute scores helps address this challenge. However, the limited availability of high-quality data often leads to weak performance of multi-objective reward functions, which can negatively impact overall performance and become the bottleneck. To address this issue, we propose a unified reward modeling framework that jointly trains Bradley-Terry (BT) single-objective and multi-objective regression-based reward functions using a shared embedding space. We theoretically establish a connection between the BT loss and the regression objective and highlight their complementary benefits. Specifically, the regression task enhances the single-objective reward function’s ability to mitigate reward hacking in challenging OOD settings, while BT-based training improves the scoring capability of the multi-objective reward function, enabling a 7B model to outperform a 70B baseline.
Extensive experimental results demonstrate that our framework significantly improves both the robustness and the scoring performance of reward models.
\end{abstract}

\section{Introduction}

Pretrained large language models perform exceptionally well on a wide range of tasks, including language understanding \cite{chang2024survey, nam2024using}, text generation \cite{zhao2023survey, wang2024comprehensive}, code synthesis \cite{jiang2024survey, nguyen2025codemmlu}, and decision‑making \cite{ye2025rational, lin2025far}. However, strong task performance alone does not ensure that these models behave safely or in alignment with human values \cite{dai2024safe, li2025data, shen2023large, zhang2025catastrophic}. To address these concerns, two main alignment methods have been developed. Supervised fine‑tuning (SFT) adjusts a base model using human‑curated prompt–response pairs to shape its behavior directly. RLHF \cite{bai2022training, ouyang2022training} follows a two‑step process: first, a proxy reward model is trained on human preference data to capture desired outcomes; second, the model is optimized against this reward using algorithms like PPO \cite{schulman2017proximal}, RLOO \cite{ahmadian2024back}, or GRPO \cite{shao2024deepseekmath}. 
By decoupling reward learning from policy learning, RLHF can leverage vast amounts of unlabeled data and generalize alignment to novel inputs, enhancing both safety and overall model capabilities.

However, RLHF is susceptible to reward hacking \cite{gao2023scaling, yang2024regularizing}, wherein the policy discovers shortcuts that maximize proxy rewards, such as by producing repetitive or formulaic content, without genuinely advancing the intended behaviors. Similar issues also arise in inference-time strategies like Best-of-N (BoN) sampling~\cite{gulcehre2023reinforced, dong2023raft, gui2024bonbon}.
Prior research has investigated several directions to address this issue. 
One line of work focuses on improving the reward function through ensemble methods \cite{coste2024reward, eisenstein2023helping, zhang2024improving, yan2024reward, rame2024warm, zhang2024mitigating}. However, they often require training multiple reward models, making them resource-intensive and less practical for real-applications. 
Another line investigates constrained policy optimization~\cite{moskovitz2024confronting, zhang2024mitigating, liu2024provably, zhang2024overcoming, laidlaw2024correlated}, but performance is often unstable due to sensitivity to hyperparameter tuning.
ODIN~\cite{chen2024odin} trains separate reward functions for quality and length; however, our results (Fig.~\ref{pre-ppo}) show that using length alone as a biasing factor fails to prevent reward hacking.
More recently, GRM\cite{yang2024regularizing, dai2025mitigating} incorporates text generation regularization into reward modeling and outperforms prior methods. However, the conflicting objectives of reward modeling and text generation cause training instability and sensitivity to the balancing weight (Appendix~\ref{appendix:hyperandinstability}). 
Moreover, most existing studies focus on in-distribution evaluations, and the effectiveness of these methods in out-of-distribution (OOD) settings, where training and test data come from different distributions, remains unexplored.

In this study, we first demonstrate through experiments that state-of-the-art methods fail when prompts used during PPO and BoN are drawn from a distribution different from the training data. This highlights a critical limitation in the generalization ability of current reward models under OOD settings.
We hypothesize that a BT model trained only on chosen/rejected labels remains biased and cannot distinguish fine‑grained quality differences. Inspired by recent multi‑objective reward modeling methods \cite{wang2024helpsteer, wang2024interpretable}, which leverage annotations for attributes like helpfulness, verbosity, and correctness, we examine their potential to mitigate reward hacking. 
By learning from multi‑dimensional supervision, multi-objective reward models (MORMs) capture nuanced distinctions in response quality and compel policy models to improve across all attributes simultaneously, making it harder to generate low‑quality outputs that nonetheless score highly (Sec. \ref{hacking}). 
While prior work focused on interpretability and steerability, the robustness of MORMs in policy learning remains underexplored.

Despite their potential, the performance of MORMs is constrained by the limited availability of large-scale, high-quality annotated data \cite{wang2024helpsteer, wang2024interpretable}.
This limitation arises either from the low-quality annotations produced by LLM-as-Judge \cite{cui2023ultrafeedback, kim2023prometheus, li2024generation, gu2024survey} or from the high-quality annotations produced by humans that are difficult to scale \cite{wang2024helpsteer2}.
A more detailed discussion of the data availability challenge is provided in Appendix \ref{appendix:motivationformultienhancing}.
Consequently, their scoring performance often falls behind that of single-objective reward models (SORMs) \cite{lambert2024rewardbench, liu2025rmbench}, which are typically trained on large-scale preference datasets with chosen/rejected labels that are easier to collect. 
A promising approach is to train a strong SORM and complement it with a MORM to enhance robustness against reward hacking. 
Though our empirical results in Sec.~\ref{hacking} verify its promise, this approach faces two key issues: \textbf{(1)} it requires two independent inference passes, making the process computationally expensive in practice; and \textbf{(2)} the weaker performance of the MORM directly degrades the quality of the aggregated output, thereby becoming a bottleneck in the overall system, as shown in Fig. \ref{pre-ppo} (d).

Thus, in this paper, we study a novel problem: \textit{how to efficiently mitigate reward hacking in challenging OOD settings using fine-grained attribute scores, without additional costly multi-attribute preference data?}
To address this problem, we propose a simple, yet effective and theoretically grounded reward modeling framework, termed the Joint Single and Multi-Objective Reward Model (\textbf{SMORM}). 
The proposed SMORM framework addresses the first challenge by requiring only a single forward pass through a shared backbone. We theoretically establish the connection between the commonly used Bradley–Terry \cite{bradley1952rank} loss for training the SORM and the regression loss used for training the MORM. Our theoretical analysis and extensive empirical results demonstrate that: \textbf{(1)} training the multi-objective head refines the embedding space such that the representations capture quality distinctions across multiple attributes, thereby \textit{enhancing the generalizability of the single-objective head and improving its robustness against reward hacking}.; and \textbf{(2)} training the single-objective head helps correct the positioning of responses in the embedding space, \textit{enabling the multi-objective head to perform competitively even with limited data}. Overall, the joint training of both heads over a shared embedding space leads to \textbf{complementary benefits}.
Using the same multi-objective dataset, SMORM enables a 7B model to outperform a 70B baseline model.
Importantly, SMORM training is flexible in that the prompt-response pairs used to train the two heads do not need to be identical.

Our primary contributions are:
\textbf{(1)} We empirically demonstrate that SOTA reward models suffer from reward hacking during PPO in OOD settings—a critical issue largely overlooked in prior work.
\textbf{(2)} We propose SMORM, an effective and theoretically grounded method to mitigate this challenge.
\textbf{(3)} We theoretically and empirically establish a connection between Bradley–Terry reward modeling and multi-objective regression-based reward modeling, revealing their complementary benefits.
\textbf{(4)} SMORM also tackles the key challenge of improving the performance of multi-objective reward model without requiring additional, hard-to-obtain annotated data.




\section{Background} 

\noindent\textbf{Bradley–Terry Single-Objective Reward Modeling.} Single-objective reward modeling typically builds on the Bradley-Terry framework \cite{bradley1952rank}, which distinguishes between a chosen response $y_c$ and a rejected response $y_r$ for a given prompt $x$. This is achieved by optimizing the following loss function:
\begin{equation}
\label{btloss}
{\min}_{\theta}\mathcal{L}_{\text{reward}}(\theta) = -\mathbb{E}_{(x,y_c,y_r) \sim \mathcal{D}_s} \left[ \log \left( \sigma \left( r_{\theta}(x, y_c) - r_{\theta}(x, y_r) \right) \right) \right],
\end{equation}
where $r_{\theta}$ is the reward model parameterized by $\theta$, $r_{\theta}(x, y)$ denotes the reward score assigned by $r_{\theta}$ for the output $y$ given the prompt $x$, and $\sigma(\cdot)$ is the sigmoid function. Minimizing this loss encourages the model to assign higher scores to outputs that are preferred by humans. 

\noindent\textbf{Multi-Objective Reward Modeling.}
In many practical settings, evaluating language model outputs requires considering multiple aspects such as correctness, coherence, and verbosity. Single reward signals often fail to capture this complexity. To address this, multi-objective reward models~\cite{wang2024helpsteer, wang2024interpretable, wang2023helpsteer, wang2024arithmetic} generate separate reward signals for different response attributes. The model is trained as:
\begin{equation}
{\min}_\psi \mathcal{L}(\psi) = \mathbb{E}_{(x, y, \mathbf{r})\sim\mathcal{D}_M} \left\| R_{\psi}(x, y) - \mathbf{r} \right\|_2^2,
\end{equation}
where $R_{\psi}$ is the multi-objective reward model and $\mathbf{r} \in \mathbb{R}^K$ denotes attribute scores (e.g., correctness, verbosity). Each dimension reflects a specific quality, enabling more interpretable and steerable evaluations. However, existing work has not examined how such models relate to reward hacking.

\noindent\textbf{Best-of-n Sampling (BoN).} Given an input $x$, BoN first draws a set \(\mathcal{Y}_{\mathrm{gen}}\) of \(n\) candidate outputs from the policy model and then selects the one that maximizes the reward‐model score. It can be applied either to improve outputs at inference time or to drive an iterative optimization procedure~\cite{gulcehre2023reinforced, dong2023raft, gui2024bonbon}:
\begin{equation}\label{eq:BoN}
y_{\mathrm{BoN}}(x) = \arg\max_{y \in \mathcal{Y}_{\mathrm{gen}}} r_\theta(x, y).
\end{equation}
\noindent\textbf{Proximal Policy Optimization (PPO).} PPO is a widely adopted method for RLHF in optimizing language models \cite{ouyang2022training, stiennon2020learning, wu2024pairwise}. Using a proxy reward model $r_\theta$, PPO refines the policy model $\pi_\phi$ by maximizing its score under the proxy reward while incorporating a KL divergence penalty:
\begin{equation}
    {\max}_{\phi} \mathcal{L}_{\text{RLHF}}(\phi) := \mathbb{E}_{x \sim S} \left[ \mathbb{E}_{y \sim \pi_\phi(\cdot|x)} \left( r_{\theta}(x, y) \right) - \lambda \cdot \mathrm{KL}\left( \pi_\phi(\cdot|x) \,\|\, \pi_\phi^{\text{ref}}(\cdot|x) \right) \right].
\end{equation}
Here, $S$ is a training set of prompts and $\lambda \ge 0$ is a KL regularization coefficient that controls how much $\pi_\phi$ deviates from the initial policy $\pi_\phi^{\text{ref}}$.


\noindent\textbf{Reward Hacking.} 
Reward hacking arises when a policy model exploits flaws in the reward function, achieving high scores by overfitting to spurious patterns rather than fulfilling the intended task \cite{fu2025reward, yang2024regularizing, gao2023scaling}. 
In PPO, this issue manifests when the policy gains higher scores from a proxy reward model but performs worse on a more reliable, human-aligned golden reward model. 
Further discussion of related work and comparisons to existing approaches is provided in Appendix~\ref{appendix:related}.


\section{SMORM: A Method to Mitigate Reward Hacking in an OOD Setting}
\label{hacking}
In this section, we conduct experiments to evaluate the robustness of existing SOTA methods and our proposed framework against reward hacking in both PPO and BoN under OOD settings, a more challenging setting largely overlooked in prior work \cite{yang2024regularizing, coste2024reward, eisenstein2023helping, zhang2024improving, yan2024reward, rame2024warm, dai2025mitigating}.
\begin{figure}[t]
    \centering
    \includegraphics[width=0.97\linewidth]{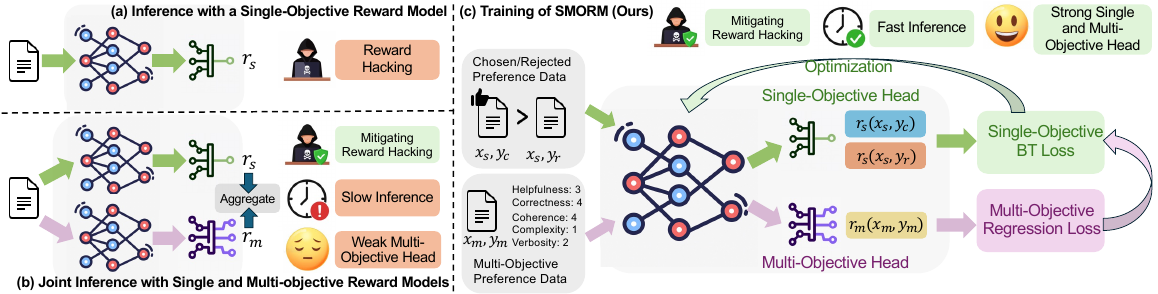}
    \caption{Illustration of SMORM training and its advantages over baseline methods.}
    \label{fig:framework}
\end{figure}
Intuitively, a multi-objective reward model (MORM) evaluates response quality across multiple dimensions, forcing the policy to balance various attributes during generation. For example, it cannot maximize the overall score by optimizing helpfulness alone while neglecting verbosity.
However, the performance of MORM is limited by the need for large-scale, high-quality multi-attribute labels, which are typically more difficult to collect than binary chosen/rejected labels \cite{wang2024helpsteer2}.
This motivates us to design an ensemble framework that complements the stronger SORM with a MORM to enhance robustness against reward hacking, as shown in Fig. \ref{fig:framework} (b).
However, this framework encounters two major issues: (1) it requires two separate inference passes, leading to significant computational overhead; and (2) the weaker performance of the MORM degrades the quality of the aggregated output, becoming a bottleneck for the system, as verified in Fig. \ref{pre-ppo}.

To overcome these limitations, we introduce a simple yet effective, theoretically grounded reward modeling framework called the Single and Multi-Objective Reward Model (\textbf{SMORM}). As shown in Fig. \ref{fig:framework} (c), SMORM jointly trains a BT single- and multi-objective reward function based on a shared embedding space. It addresses the issue of efficiency by requiring only a single forward pass through the backbone. Furthermore, training the single-objective function along the chosen/rejected dimension helps shape the embedding space and correct the positioning of samples within it, thereby enabling the multi-objective head to achieve competitive performance with limited data. We will theoretically and empirically justify this effect in Section~\ref{scoring}.

Specifically, given a pre-trained decoder-only LLM without the original output linear layer as the feature extractor $f_\theta$. We pass $x \oplus y$, the concatenation of $x$ and $y$, through the decoder layers and take the hidden state of the final decoder layer as a $d$-dimensional feature. 
On top of \( f_\theta \), we attach two linear heads: a single-objective head with weights \( \mathbf{w}_S \in \mathbb{R}^{d \times 1} \), which outputs a scalar rating, and a multi-objective head with weights \( \mathbf{w}_M \in \mathbb{R}^{d \times k} \), which produces a \( k \)-dimensional vector of attribute scores.
Given $\mathcal{D}_S=\{x_s,y_c,y_r\}$ as the chosen-rejected preference dataset and $\mathcal{D}_M=\{x_m,y_m,\mathbf{r}\}$ as the multi-attribute preference dataset, SMORM is trained with the following loss function:
\begin{equation}
\min_{\theta, \mathbf{w}_S, \mathbf{w}_M} \, -\mathbb{E}_{\mathcal{D}_S} \left[\log\sigma(\mathbf{w}_S^\top(( f_\theta(x_s, y_c)-f_\theta(x_s , y_r))))\right]  +
\mathbb{E}_{\mathcal{D}_M} \left\| \mathbf{w}_M^\top f_\theta(x_m , y_m) - \mathbf{r} \right\|_2^2
\end{equation} 
The first head with weight $\mathbf{w}_S$ outputs a score along the chosen/rejected dimension, while the second head with weight $\mathbf{w}_M $ outputs scores on multiple attributes (e.g. helpfulness, coherence and verbosity). 
SMORM supports multiple inference strategies: SMORM-F uses only the first head to produce the reward score, SMORM-L computes the mean of the scores from the second head, and SMORM-M averages the scores from both the first and second heads. 


\noindent\textbf{Experimental Setup.} 
For training SMORM, we use \texttt{Skywork80K} \cite{liu2024skywork} as $\mathcal{D}_S$ and \texttt{Helpsteer2} \cite{wang2024helpsteer2} as $\mathcal{D}_M$. 
To train the multi-attribute head using only label information without extra data or domain knowledge, we filter $\mathcal{D}_M$ to retain only samples that also appear in $\mathcal{D}_S$.
We compare \texttt{SMORM} against several baselines, including: 
(1) \textit{Baseline Classifier}, trained using the original reward loss defined in Eq.~\ref{btloss}; 
(2) \textit{ODIN}, which trains two separate reward functions for quality and length \citep{chen2024odin}.
(3) \textit{Baseline SM}, which trains a baseline SORM and MORM separately and aggregates their results during inference
and (4) \textit{GRM} \cite{yang2024regularizing}, which incorporates supervised fine-tuning (SFT) regularization.
We use \texttt{gemma-2B-it} \cite{team2024gemma} as the backbone model for all methods.
In PPO experiment following \cite{yang2024regularizing}, we downsample 20K samples from the \texttt{Unified-Feedback} dataset to optimize the PPO policy, reserving an additional 1K samples as a held-out test set for evaluation. 
Following prior work \cite{gao2023scaling, coste2023reward, yang2024regularizing}, we perform BoN sampling on this evaluation set, selecting the best of $n$ responses per prompt based on proxy model scores. The selected responses are then evaluated using the gold reward model, and gold scores are averaged across all prompts to assess true quality. We vary KL divergence from 0 to 5 by adjusting $n$ from 1 to 405, using the relation $\mathrm{KL}_{\mathrm{BoN}} = \log n - \frac{n - 1}{n}$. 
For both experiments, \texttt{gemma-2B-it} serves as the policy model, and the gold reward model\footnote{\href{https://huggingface.co/Ray2333/reward-model-Mistral-7B-instruct-Unified-Feedback}{Ray2333/reward-model-Mistral-7B-instruct-Unified-Feedback}} is a 7B human-preference model fine-tuned on the full \texttt{Unified-Feedback} dataset.
\begin{figure}[ht]
    \centering
    \includegraphics[width=\linewidth]{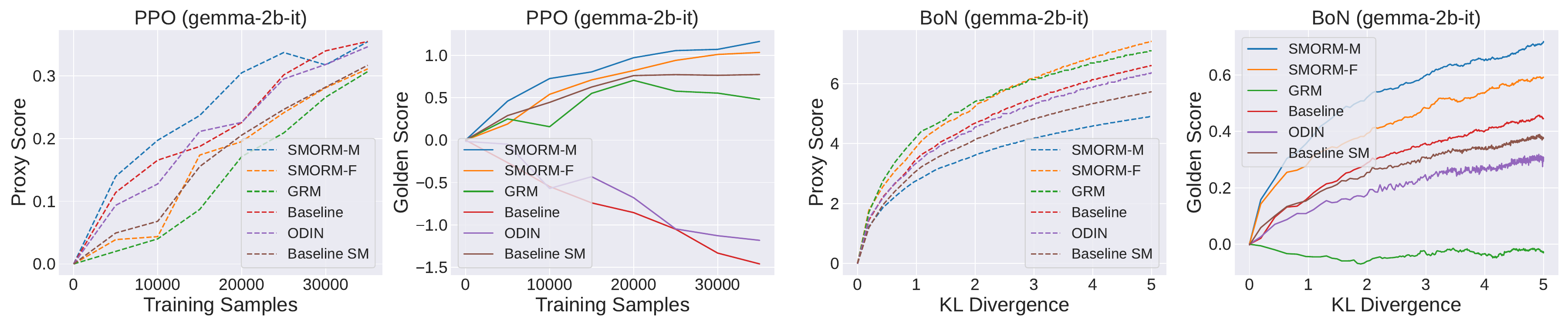}
    \begin{minipage}{0.24\textwidth}
        \small
        \centering
        \textbf{(a)}
    \end{minipage}%
    \begin{minipage}{0.24\textwidth}
        \small
        \centering
        \textbf{(b)}
    \end{minipage}
    \begin{minipage}{0.25\textwidth}
        \small
        \centering
        \textbf{(c)}
    \end{minipage}
    \begin{minipage}{0.25\textwidth}
        \small
        \centering
        \textbf{(d)}
    \end{minipage}
    \caption{Proxy and gold scores from (a)(b) PPO and (c)(d) BoN experiments under the OOD setting, using gemma-2B-it as the base model. All rewards are normalized to start from 0.}
    \label{pre-ppo}
\end{figure}


The results of the PPO and BoN experiments are presented in Fig.~\ref{pre-ppo} (a)(b) and (c)(d), respectively. From these figures, we observe the following:
\textbf{(1) Baselines Fail to Mitigate Reward Hacking.} We observe that although GRM demonstrates increased golden scores during PPO, its average gold scores decrease as KL divergence increases in the BoN setting, suggesting reward hacking. ODIN shows a modest increase in golden scores under BoN, but a decrease during PPO, indicating that addressing length alone as a bias is insufficient to mitigate reward hacking effectively.
\textbf{(2) Incorporating Multi-Attribute Scores Helps Mitigate Reward Hacking in OOD Setting.} In both experimental settings, we find that the Baseline SM outperforms both GRM and ODIN, exhibiting a steady increase in golden scores as the proxy score increases.
\textbf{(3) Weak Multi-Objective Reward Functions Become a Bottleneck.} While Baseline SM shows promise, its performance in the BoN experiment is even worse than that of the baseline single-objective reward function. This suggests that weak multi-objective reward function can be detrimental and become a bottleneck in the overall system.
\textbf{(4) SMORM Shows Superior Performance.} Both SMORM-F and SMORM-M significantly outperform all baselines across both experiments. Notably, SMORM-F performs comparably to SMORM-M.
To theoretically justify this phenomenon, we propose the following theorem:

\begin{theorem}[Implicit Multi-Attribute Effect]
\label{prop:implicit_multi_effect_detailed}
Let a reward model be trained under the SMORM framework, and suppose the following conditions hold:
(1) Bounded features: There exists $B<\infty$ such that $\|f_\theta(x,y)\|\le B$ for every $(x,y)$.
(2)
Positive-definite covariances: let
$
  f_c=f_\theta(x_s,y_c),\;
  f_r=f_\theta(x_s,y_r),\;
  f_m =f_\theta(x_m,y_m).$
        \(
          \Sigma_S:=\E_{\mathcal{D}_S}[(f_c-f_r)(f_c-f_r)^\top]
          \;
          \text{and}
          \;
          \Sigma_M:=\E_{\mathcal{D}_M}[f_m\,f_m^\top]
        \)
        are positive-definite matrices.
(3) Positive correlation: Let
        $\mu_S:=\E_{(x_s,y_c,y_r)\sim\mathcal{D}_S}[f_\theta(x_s,y_c)-f_\theta(x_s,y_r)]$
        and let  
        $C_M:=\E_{(x_m,y_m,r)\sim\mathcal{D}_M}[f_\theta(x_m,y_m)\,\mathbf{r}^\top]\in\R^{d\times K}$.
        Then $\alpha:=\mu_S^\top\!\Sigma_M^{-1}C_M\ \text{has non-negative sum, i.e.}\;
            \mathbf 1^\top\alpha \;\ge\;0.$
As the optimization of both reward heads converge to their population minimizers, there exist constants $c=\frac{\mathbf{1}^\top\alpha}{K\,\bigl(\mu_S^\top\,\Sigma_S^{-1}\,\mu_S\bigr)}$ and $\varepsilon\ge0$—depending only
on $B$ and second-order moments—such that for every pair \((x,y)\):
\begin{equation}
\resizebox{0.83\linewidth}{!}{$
r_m(x,y)\;=\;\frac1K\sum_{i=1}^K w_{M,i}^\top f_\theta(x,y)
  \;\;\ge\;\;
  c\,\bigl(w_S^\top f_\theta(x,y)\bigr)\;-\;\varepsilon=cr_s(x,y)\;-\;\varepsilon.
  $}
\end{equation}
\end{theorem}
We provide detailed justifications of the assumptions and the proofs in the Appendix \ref{prooftheorem1}.
The theorem is proved under the assumption that the aggregated attribute scores in the multi-objective preference data are positively correlated with the score provided by a pretrained single-objective reward function $\mathbf{w}_S$.
This is well-justified: \(\mathbf{w}_S\) can be trained on large-scale, high-quality preference data and thus provides reliable scoring. 
Moreover, multi-objective annotations typically follow the principle that chosen responses should have higher aggregated scores than rejected ones, and thus a positive correlation can be reasonably expected. 
Notably, this assumption is general and does not require the training data for the two heads to share the same prompt–response pairs or domain.

From this theorem, we derive two conclusions: 
\textbf{(1)} If a response achieves a high single-objective score \(r_s \ge \tau\), its multi-attribute average score \(r_m\) is lower-bounded by \(c\,\tau - \varepsilon\). This result implies that even when the multi-attribute head is ignored, a high single-objective score alone ensures a respectable level of fine-grained quality. 
This explains why policies trained with SMORM-F attain performance comparable to those trained with SMORM-M.
\textbf{(2)} For any two responses \(y_A\) and \(y_B\) to the same prompt \(x\), if $r_s(x, y_A) > r_s(x, y_B),$ then it follows that $c\,r_s(x, y_A) - \varepsilon > c\,r_s(x, y_B) - \varepsilon.$
That is, a higher single-objective score implies a strictly higher lower bound on the corresponding multi-attribute score $r_m$. This is desirable, as it allows us to train a strong single-objective function using abundant preference data and use it to guide the preference decisions of the multi-objective function, without requiring the costly annotation of fine-grained multi-attribute labels. The empirical results in Sec.~\ref{exp:rewardmodeling} support this conclusion.

\section{Superior Scoring Performance of SMORM}
\label{scoring} 

We have empirically and theoretically demonstrated that SMORM effectively mitigates reward hacking, even under challenging OOD scenarios. Although Theorem~\ref{prop:implicit_multi_effect_detailed} establishes that SMORM can implicitly leverage a stronger single-objective reward function to guide the multi-objective one, a potential concern remains: joint training may degrade the performance of that single-objective function. To address this concern, we begin this section by introducing a lemma that connects Bradley–Terry reward modeling with multi-objective reward regression. We then present a theorem demonstrating that SMORM improves the performance of both reward functions by learning a more effective feature extractor.
Unlike ODIN \cite{DBLP:conf/icml/Chen0CS0GHSC24}, which trains reward functions for both quality and length using the BT loss, the interaction between BT loss and MSE regression loss when sharing a common embedding space is not straightforward to characterize. Therefore, we propose the following lemma to bridge this gap. 

\begin{lemma}[Pairwise Preference Error to MSE Loss]
\label{lemma1}
Let \( y_A, y_B \) be a pair of responses. Assume \( g_s(y) \) is the ground truth score and \( r_s(y) \) is the predicted score under a Bradley–Terry model. Then:
\[
\P(y_A \succ y_B) = \sigma\big(r_s(y_A) - r_s(y_B)\big), \quad 
\P^\star(y_A \succ y_B) = \sigma\big(g_s(y_A) - g_s(y_B)\big),
\]
where \( \sigma(t) = \frac{1}{1 + e^{-t}} \). The expected preference error satisfies:
\begin{equation}
\E_{\mathcal{D}_S} \left| \P(y_A \succ y_B) - \P^\star(y_A \succ y_B) \right|
\le 
\frac{1}{4} \E_{\mathcal{D}_S} \left( \sqrt{2\, \text{MSE}(r_s)} \right),
\end{equation}
with \(\text{MSE}(r_s) = \big(r_s(y) - g_s(y)\big)^2\).
Similarly, for a multi-objective reward model with predicted score \( r_m \) and ground truth \( g_m \), let $e_m = r_m(v_A) - r_m(v_B)$ and $e_m^\star = g_m(v_A) - g_m(v_B)$. 
Then the error is bounded as:
\begin{equation}
\E_{\mathcal{D}_M} \left| e_m - e_m^\star \right| \le \E_{\mathcal{D}_M}\left(\sqrt{2\, \text{MSE}(r_m)}\right).
\end{equation}
\end{lemma}
Lemma~\ref{lemma1} shows that by keeping the mean squared error of the predicted scores small, we automatically bound the expected pairwise preference error for both the Bradley–Terry single-objective head and the multi-attribute head. With this lemma, we propose the following theorem:  

\begin{theorem}
\label{theorem2}
Under the same assumptions as in Theorem~\ref{prop:implicit_multi_effect_detailed} and assuming that the feature extractor $f_\theta$ is differentiable, let \(\widehat{\theta}\) denote the maximum likelihood estimator (MLE) of the ground truth optimal parameter \(\theta^\star\). Let \(\widehat{\theta}_{\text{s}}\) and \(\widehat{\theta}_{\text{m}}\) denote the maximum likelihood estimators of the single- and multi-objective reward functions, respectively.
Define
\(M_S(y) = \mathbf{w}_S^\top f_{\theta^\star}(y),\) \(M_M(y) = \mathbf{w}_M^\top f_{\theta^\star}(y)\). Then, for a response \(y\), the mean squared error (MSE) of the predicted reward can be approximated as:
\[ \label{eq:thm-MSE0}
\resizebox{0.99\linewidth}{!}{$
\operatorname{MSE}_S \approx \nabla_\theta M_S(y)^{\top} \operatorname{Cov}\left( \widehat{\theta}_s \right) \nabla_\theta M_S(y) + \sigma_{00}, \operatorname{MSE}_M \approx \nabla_\theta M_M(y)^{\top} \operatorname{Cov}\left( \widehat{\theta}_m \right) \nabla_\theta M_M(y) + \sigma_{00},$}
\]
where $\sigma_{00}$ is the intrinsic randomness in the label. Moreover, SMORM yields lower asymptotic MSE for both the single- and multi-objective heads compared to training either head alone:
\begin{equation} \label{eq:thm-MSE0-inequality}
\operatorname{MSE}_S^{\text{SMORM}} 
<
\operatorname{MSE}_S^{\text{single}},\quad
\operatorname{MSE}_M^{\text{SMORM}} 
<
\operatorname{MSE}_M^{\text{multi}}
\end{equation}
\end{theorem}
Detailed proofs of the above lemma and theorem are provided in Appendix~\ref{appendix:theory}. By Lemma~\ref{lemma1} and Theorem~\ref{theorem2}, a reduction in MSE directly tightens the bound on pairwise preference prediction error. Thus, SMORM training yields an asymptotically more accurate response quality estimator (in terms of MSE) than either the single-head or multi-head training procedure.
We provide an analysis of why the baseline MORM fails on RM-Bench and how SMORM improves performance in Appendix~\ref{appendix:rmbench}.

\section{Experiments}
\label{experiments}
In this section, we present a comprehensive evaluation of SMORM. We conduct experiments under both in-distribution and out-of-distribution settings to assess how SMORM improves scoring capability and mitigates reward hacking. 
Our results demonstrate the following advantages of SMORM: 
\textbf{(1) Enhancing scoring performance}. SMORM significantly improves the scoring capability of the multi-objective reward function without requiring additional multi-attribute preference data, which are often difficult to collect (Sec.~\ref{exp:rewardmodeling}).
\textbf{(2)} \textbf{Mitigating reward hacking}. SMORM substantially enhances the robustness of the single-objective reward function against reward hacking in both ID and challenging OOD settings (Sec. \ref{exp:rlhf}), while also improving its scoring performance (Sec.~\ref{exp:rewardmodeling}).
\textbf{(3) Flexible training}. SMORM does not require the single-objective dataset $\mathcal{D}_S$ and the multi-objective dataset $\mathcal{D}_M$ to share the same prompt-response pairs, allowing for a more flexible training process.

\subsection{Experimental Setup} 

\noindent\textbf{Datasets and Benchmarks.}
In Section~\ref{exp:rewardmodeling}, we use the \texttt{Unified-Feedback} dataset as the single-objective preference dataset \(\mathcal{D}_S\), one of the largest collections of pairwise human preferences. To assess robustness across data scales, we evaluate two settings: one with 400K samples from \(\mathcal{D}_S\) and another with 40K. For the multi-objective dataset \(\mathcal{D}_M\), we use \texttt{UltraFeedback}~\cite{cui2023ultrafeedback} in the 400K setup, which includes 240K GPT-annotated prompt--response pairs with fine-grained attribute scores. In the 40K setup, we use \texttt{HelpSteer2}~\cite{wang2024helpsteer2}, a high-quality human-annotated dataset with 20K samples. By varying both the size and source of \(\mathcal{D}_M\), we demonstrate that our method consistently improves multi-objective reward modeling across datasets of varying scale and quality.
SORM is trained only on \(\mathcal{D}_S\), MORM on \(\mathcal{D}_M\), and SMORM jointly on both. The trained reward models are evaluated on RewardBench \cite{lambert2024rewardbench} and RM-Bench \cite{liu2025rmbench}.
In the RLHF experiments (Section~\ref{exp:rlhf}), we consider both in-distribution (ID) and out-of-distribution (OOD) settings. For ID, we downsample 20K samples from \texttt{Unified-Feedback} as \(\mathcal{D}_S\) to train reward models and the PPO policy. For OOD, we use \texttt{Skywork80K}~\cite{liu2024skywork} as \(\mathcal{D}_S\). In both cases, \texttt{HelpSteer2} serves as \(\mathcal{D}_M\), and 1K samples from \texttt{Unified-Feedback} are reserved for policy evaluation.

\noindent\textbf{Base Models.} Following \cite{yang2024regularizing}, we adopt gemma-2B-it \cite{team2024gemma} and Mistral-7B-Instruct-v0.2 \cite{jiang2023mistral7b} as the base models for preference learning. For the RLHF experiments, gemma-2B-it serves as the policy model in both the BoN and PPO settings. The gold reward model is a 7B human preference model fine-tuned on the entire Unified-Feedback dataset.

\noindent\textbf{Baselines.}
We adopt representative and SOTA baselines:
(1) Baseline Classifier, trained using the original reward loss as defined in Eq. \ref{btloss};
(2) Margin~\cite{touvron2023llama}, which augments the original reward loss with an additional margin term;
(3) Label Smooth~\cite{wang2024secrets}, which addresses overfitting by penalizing overconfident predictions;
(4) Ensemble, which combines the outputs of three reward models by taking either the average or the minimum score \cite{coste2023reward}; and
(5) GRM\cite{yang2024regularizing} with two types of regularization: GRM w/ dpo and GRM w/ sft.
Details on baselines and implementation are in Appendix~\ref{appendix:implementation}.

\noindent\textbf{RLHF.} The training and evaluation pipeline of PPO and BoN follow the setting in Sec. \ref{hacking}.


\subsection{Evaluation on Reward Modeling}
\label{exp:rewardmodeling}
\begin{table*}[h]
\centering
\caption{Comparison of SMORM-F and baselines on RewardBench. Baselines results from \cite{yang2024regularizing}.}
\resizebox{\textwidth}{!}{%
\begin{tabular}{l|ccccc|ccccc}
\toprule
\multirow{2}{*}{\textbf{Reward model}} & \multicolumn{5}{c|}{$\mathcal{D_S}/\mathcal{D}_M$: UnifiedFeedback 400k/UltraFeedback} & \multicolumn{5}{c}{$\mathcal{D_S}/\mathcal{D}_M$: UnifiedFeedback 40k/HelpSteer2} \\
\cmidrule{2-11}
& \textbf{Chat} & \textbf{Chat-Hard} & \textbf{Safety} & \textbf{Reasoning} & \textbf{Avg} & \textbf{Chat} & \textbf{Chat-Hard} & \textbf{Safety} & \textbf{Reasoning} & \textbf{Avg} \\
\midrule
Baseline (Single)             
& 95.5 & 38.0 & 73.8 & 65.3 & 68.2
& 94.7 & 37.5 & 66.2 & 58.4 & 64.2 \\
Baseline + margin                 
& 95.8 & 38.4 & 73.9 & 72.5 & 70.2
& 97.2 & 37.5 & 56.8 & 72.7 & 66.1 \\
Label smooth           
& 94.4 & 37.3 & 73.2 & 77.4 & 70.6
& 91.6 & 39.0 & 53.8 & 60.2 & 61.1 \\
Ensemble               
& 98.0 & 37.5 & 77.3 & 71.3 & 71.0
& 96.1 & 38.2 & 58.8 & 67.6 & 65.2 \\
GRM (linear) w/ dpo 
& 96.7 & 39.0 & 76.4 & 68.5 & 70.2
& 94.7 & 38.4 & 62.5 & 51.2 & 61.7 \\
GRM (linear) w/ sft          
& 96.1 & 40.1 & 80.3 & 69.3 & 71.5
& 94.7 & 40.8 & 65.4 & 77.0 & 69.5 \\
GRM w/ dpo           
& 95.8 & 40.1 & 78.7 & 66.2 & 70.2
& 92.5 & 39.9 & 72.5 & 61.4 & 66.6 \\
GRM w/ sft                   
& 97.8 & 42.1 & 77.9 & 65.2 & 70.8
& 94.1 & 41.9 & 69.5 & 61.5 & 66.8 \\
\rowcolor{gray!20}
SMORM-F  
& 96.1 & 45.5 & 78.8 & 70.9 & \textbf{72.8}
& 96.1 & 44.1 & 81.1 & 62.7 & \textbf{71.0} \\
\bottomrule
\end{tabular}
}
\label{tab:sorm_combined}
\end{table*}
\noindent\textbf{Comparison to Single-Objective Reward Model.} Table~\ref{tab:sorm_combined} compares SMORM-F with baseline SORMs using gemma-2B-it as the base model on RewardBench. Results using Mistral-7B-Instruct as the base model are provided in Appendix \ref{appendix:additionalsingle}.
We observe that SMORM-F consistently achieves the highest average performance across all experimental settings. This finding supports Theorem~\ref{theorem2}, which establishes that our joint training framework leads to superior performance compared to training single-objective reward functions in isolation. Additional results on evaluating SMORM-F in OOD settings are provided in Appendix~\ref{appendix:oodevaluation}.

\noindent\textbf{Comparison to Multi-Objective Reward Model.} Table~\ref{tab:morm_combined} reports the results of comparing SMORM-L with baseline MORMs on RewardBench and RM-Bench. We find that SMORM significantly improves the performance of its multi-objective head.
For instance, using Mistral-7B-Instruct as the base model and 40k samples from \texttt{UnifiedFeedback}, SMORM achieves a score of \textbf{13.9} higher on RewardBench and \textbf{12.4} higher on RM-Bench.
These findings validate Theorem~\ref{prop:implicit_multi_effect_detailed}, showing that BT modeling helps correct the positioning of responses such that the multi-objective score is lower-bounded by the corresponding single-objective score. They also support Theorem~\ref{theorem2}, confirming that our joint training framework yields reward models that outperform their baseline counterparts.
\begin{table*}[h]
    \centering
    \caption{Comparison of SMORM-L and baseline MORM on RewardBench and RM-Bench.}
    \resizebox{0.95\textwidth}{!}{
    \begin{tabular}{lcccccc}
        \toprule
        \textbf{Reward model}    & \textbf{Chat}   & \textbf{Chat Hard} & \textbf{Safety} & \textbf{Reasoning} & \textbf{RewardBench} & \textbf{RM-Bench}\\
        \midrule
        \multicolumn{7}{c}{Base Model: Gemma 2b it, $\mathcal{D_S}/\mathcal{D}_M$: UnifiedFeedback 400k/UltraFeedback} \\
        \midrule
        Baseline (Multi)  & 64.2 & 50.0 & 46.1 & 42.3 & 50.6 & 50.2\\
        SMORM-L           & 90.8 & 48.3 & 61.5 & 53.7 & \textbf{63.6} & \textbf{55.1}\\
        \midrule
        \multicolumn{7}{c}{Base Model: Gemma 2b it, $\mathcal{D_S}/\mathcal{D}_M$: UnifiedFeedback 40k/HelpSteer2} \\
        \midrule
        Baseline (Multi)  & 84.1 & 39.0 & 57.4 & 42.6 & 55.8 & 51.8\\
        SMORM-L           & 94.9 & 39.3 & 75.8 & 51.4 & \textbf{65.4} & \textbf{54.1}\\
        \midrule
        \multicolumn{7}{c}{Base Model: Mistral 7b Instruct, $\mathcal{D_S}/\mathcal{D}_M$: UnifiedFeedback 400k/UltraFeedback} \\
        \midrule
        Baseline (Multi)  & 95.5 & 65.1 & 68.6 & 75.9 & 76.3 & 57.2\\
        SMORM-L           & 97.8 & 61.0 & 86.4 & 77.7 & \textbf{80.7} & \textbf{62.6}\\
        \midrule
        \multicolumn{7}{c}{Base Model: Mistral 7b Instruct, $\mathcal{D_S}/\mathcal{D}_M$: UnifiedFeedback 40k/HelpSteer2} \\
        \midrule
        Baseline (Multi)  & 71.8 & 60.1 & 54.3 & 78.0 & 66.0 & 52.0\\
        SMORM-L           & 94.4 & 61.8 & 83.6 & 79.7 & \textbf{79.9} & \textbf{64.4}\\
        \bottomrule
    \end{tabular}
    }
    \label{tab:morm_combined}
\end{table*}

\noindent\textbf{Comparison to Advanced Reward Models.} 
We follow the approach of \cite{wang2024interpretable} and train a gating network that automatically assigns weights to each attribute during inference. 
We train our models using both 7B and 8B variants. The 7B model is initialized from Mistral-7B-Instruct-v0.2~\citep{jiang2023mistral7b} and trained on $\mathcal{D}_S/\mathcal{D}_M$: \texttt{UnifiedFeedback (40K)} and \texttt{HelpSteer2 (20K)}, respectively. The 8B model is initialized from Llama-3.1-8B-Instruct~\citep{grattafiori2024llama} and trained on $\mathcal{D}_S/\mathcal{D}_M$: \texttt{Skywork80K} \cite{liu2024skywork} and \texttt{HelpSteer2 (20K)}. In both cases, the gating network is trained using the corresponding $\mathcal{D}_S$ with BT loss.
We compare our models against advanced large multi-objective reward models also trained on \texttt{HelpSteer2}, where the weight for each attribute is also tuned~\citep{wang2024helpsteer2}. Additionally, we include a comparison with ArmoRM-Llama3-8B-v0.1~\citep{wang2024interpretable}, a multi-objective reward model trained on significantly more extensive preference data. 
Specifically, the multi-objective component of ArmoRM is trained on 585.4K samples~\citep{wang2024interpretable}, compared to only 20K for ours. 
Its gating network is trained on 1,004.4K samples, while our 7B and 8B models use only 40K and 80K samples, respectively.
Table~\ref{tab:comparetolarge} presents the comparative results. While our 7B model underperforms compared to the 340B model, it outperforms the 70B model. Notably, our 8B model matches the performance of ArmoRM-Llama3-8B-v0.1—despite the latter being trained on \textbf{15.9$\times$} more data. 
These results demonstrate the effectiveness of our SMORM framework in leveraging BT reward modeling to enhance the performance of multi-objective reward modeling.
\begin{table}[h]
    \centering
    \caption{Comparison with Advanced Multi-Objective Reward Models on RewardBench.}
    \resizebox{\textwidth}{!}{
    \begin{tabular}{c|ccccccc}
    \toprule
    \textbf{Reward model}   & \textbf{Size of $\mathcal{D}_M$} &\textbf{Model Size}& \textbf{Chat}   & \textbf{Chat Hard} & \textbf{Safety} & \textbf{Reasoning} & \textbf{Avg}\\
        \midrule
         Nemotron-4-340B-RM \cite{wang2024helpsteer2} & 20K&340B&92.0&95.8 &87.1 &91.5 & 93.7 \\
         ArmoRM-Llama3-8B-v0.1 \cite{wang2024interpretable}  &585.4K&8B& 96.9 & 76.8 & 90.5 & 97.3 & 90.4 \\
         Llama-3-70B-RM \cite{wang2024helpsteer2}  &20K&70B& 91.3 &80.3 &92.8 &90.7&88.8 \\
         \midrule
         SMORM-L 7B (Ours) &20K &7B&95.0 & 80.5 & 91.6 & 89.0 & 89.0 \\
         SMORM-L 8B (Ours) &20K & 8B & 94.7 & 85.1 & 90.5 & 91.3 & 90.4 \\        
         \bottomrule
    \end{tabular}
    }
    \label{tab:comparetolarge}
\end{table}

\subsection{Evaluation on RLHF}
\label{exp:rlhf}
\noindent\textbf{In-Distribution Setting.} Figures~\ref{fig:idppo} and~\ref{fig:idbon} present the results of PPO and BoN under the in-distribution setting.
For \textit{PPO}, we observe that the gold score of the baseline reward function increases slowly in the early stages of training but subsequently declines, while the proxy score continues to rise. This pattern indicates reward overoptimization. In contrast, GRM exhibits a rapid initial increase in gold score, followed by a similar decline trend. Notably, both SMORM-F and SMORM-M show a consistent increase in gold score throughout training. SMORM-F achieves performance comparable to SMORM-M, which further supports Theorem~\ref{prop:implicit_multi_effect_detailed} regarding the implicit benefit of joint training. In the \textit{BoN} experiments, when using \texttt{gemma-2B-it} as the base model, we find that the gold score of the baseline reward function decreases as the KL divergence increases. For GRM, the gold score increases slowly and then plateaus. 
We hypothesize that this is because GRM is designed to limit adversarial exploitation, aiming to prevent models from obtaining high rewards through adversarial inputs, but it lacks theoretical guarantees for minimizing pairwise preference error.
In contrast, both SMORM-F and SMORM-M exhibit a consistent increase in gold scores across KL ranges, demonstrating the robustness and effectiveness of our framework in the in-distribution setting. 
\begin{figure}[h]
    \centering
    \includegraphics[width=\linewidth]{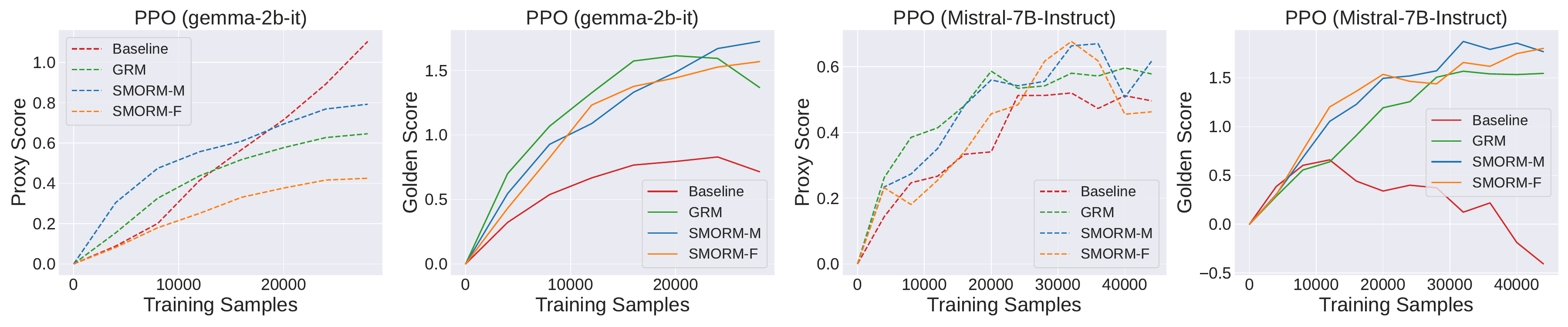}
    \begin{minipage}{0.24\textwidth}
        \small
        \centering
        \textbf{(a)}
    \end{minipage}%
    \begin{minipage}{0.24\textwidth}
        \small
        \centering
        \textbf{(b)}
    \end{minipage}
    \begin{minipage}{0.25\textwidth}
        \small
        \centering
        \textbf{(c)}
    \end{minipage}
    \begin{minipage}{0.25\textwidth}
        \small
        \centering
        \textbf{(d)}
    \end{minipage}
    \vskip -0.5em
    \caption{Proxy scores and gold scores of PPO experiments for reward model based on (a)(b) gemma2b-it and (c)(d) Mistral-7B-Instruct. All rewards are normalized to start from 0.}
    \label{fig:idppo}
\end{figure}
\begin{figure}[h]
    \centering
    \includegraphics[width=\linewidth]{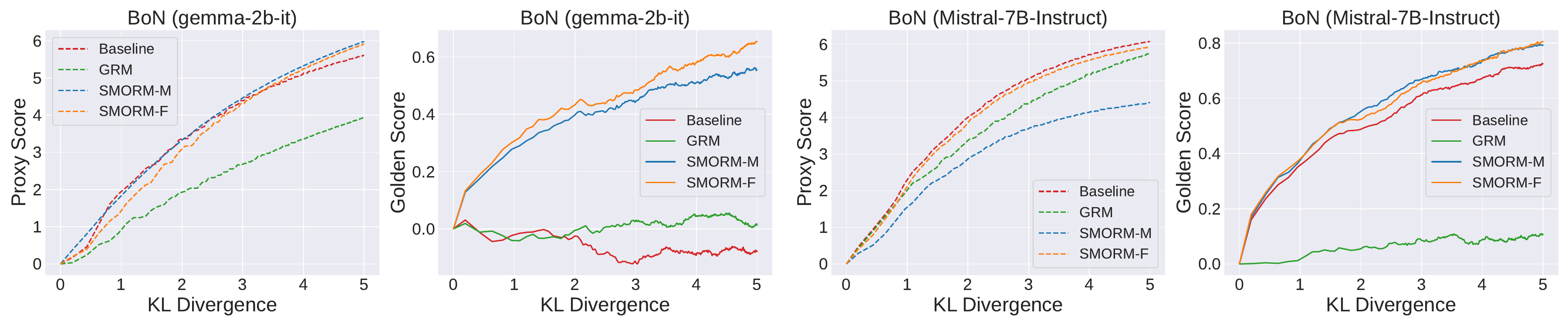}
    \begin{minipage}{0.24\textwidth}
        \small
        \centering
        \textbf{(a)}
    \end{minipage}%
    \begin{minipage}{0.24\textwidth}
        \small
        \centering
        \textbf{(b)}
    \end{minipage}
    \begin{minipage}{0.25\textwidth}
        \small
        \centering
        \textbf{(c)}
    \end{minipage}
    \begin{minipage}{0.25\textwidth}
        \small
        \centering
        \textbf{(d)}
    \end{minipage}
    \caption{Proxy scores and gold scores of BoN experiments for base models of (a)(b) gemma-2b-it and (c)(d) Mistral-7B-Instruct. Rewards are normalized to start from 0.}
    \label{fig:idbon}
\end{figure}

\noindent\textbf{Out-of-Distribution Setting.} 
Figure~\ref{fig:oodmistral} presents PPO and BoN results under the OOD setting using Mistral-7B-Instruct as the base model. Results with gemma-2B-it were shown earlier in Section~\ref{hacking}. While the baseline reward function and GRM do not exhibit clear reward hacking with the stronger model in the OOD setting, the performance gap between SMORM and GRM becomes more pronounced compared to the in-distribution results in Figure~\ref{fig:idppo}. This underscores the limitations of existing methods and demonstrates the robustness of SMORM, especially under challenging OOD conditions.
Additional results comparing PPO-optimized models are provided in Appendix~\ref{appendix:alignment}.
\begin{figure}[h]
    \centering
    \includegraphics[width=\linewidth]{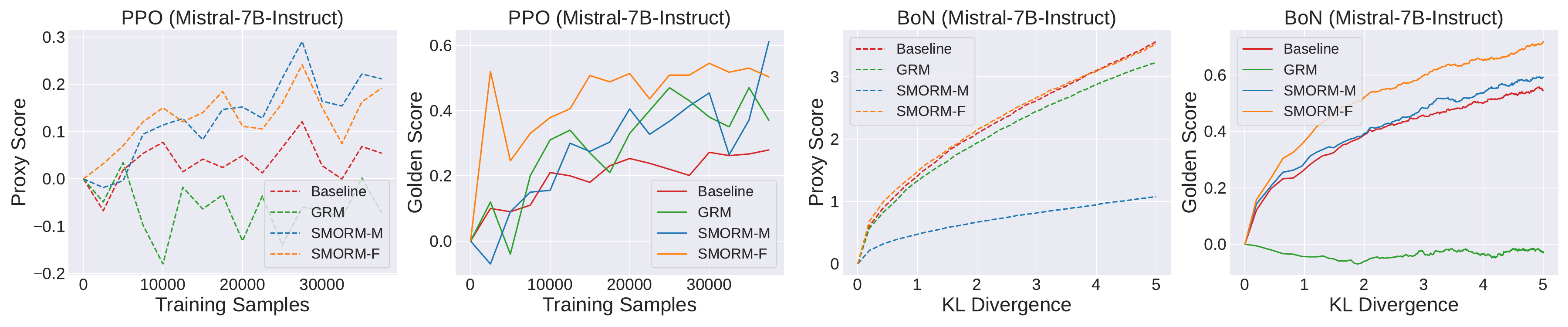}
    \begin{minipage}{0.24\textwidth}
        \small
        \centering
        \textbf{(a)}
    \end{minipage}%
    \begin{minipage}{0.24\textwidth}
        \small
        \centering
        \textbf{(b)}
    \end{minipage}
    \begin{minipage}{0.25\textwidth}
        \small
        \centering
        \textbf{(c)}
    \end{minipage}
    \begin{minipage}{0.25\textwidth}
        \small
        \centering
        \textbf{(d)}
    \end{minipage}
    \caption{Proxy and gold scores from (a)(b) PPO and (c)(d) BoN experiments under the OOD setting.}
    \label{fig:oodmistral}
\end{figure}

\section{Conclusion}
In this paper, we propose an effective and theoretically grounded method that jointly trains a Bradley–Terry reward function and a multi-objective reward function using a shared embedding space. Our theoretical analysis and extensive empirical results demonstrate that this joint training approach enhances the robustness of the BT function against reward hacking and significantly improves the scoring performance of the multi-objective function, revealing their complementary benefits.

\newpage

{
\newpage
\bibliographystyle{unsrt}
\bibliography{ref}

\begin{thebibliography}{10}

\bibitem{chang2024survey}
Yupeng Chang, Xu~Wang, Jindong Wang, Yuan Wu, Linyi Yang, Kaijie Zhu, Hao Chen, Xiaoyuan Yi, Cunxiang Wang, Yidong Wang, et~al.
\newblock A survey on evaluation of large language models.
\newblock {\em ACM transactions on intelligent systems and technology}, 15(3):1--45, 2024.

\bibitem{nam2024using}
Daye Nam, Andrew Macvean, Vincent Hellendoorn, Bogdan Vasilescu, and Brad Myers.
\newblock Using an llm to help with code understanding.
\newblock In {\em Proceedings of the IEEE/ACM 46th International Conference on Software Engineering}, pages 1--13, 2024.

\bibitem{zhao2023survey}
Wayne~Xin Zhao, Kun Zhou, Junyi Li, Tianyi Tang, Xiaolei Wang, Yupeng Hou, Yingqian Min, Beichen Zhang, Junjie Zhang, Zican Dong, et~al.
\newblock A survey of large language models.
\newblock {\em arXiv preprint arXiv:2303.18223}, 1(2), 2023.

\bibitem{wang2024comprehensive}
Fali Wang, Zhiwei Zhang, Xianren Zhang, Zongyu Wu, Tzuhao Mo, Qiuhao Lu, Wanjing Wang, Rui Li, Junjie Xu, Xianfeng Tang, et~al.
\newblock A comprehensive survey of small language models in the era of large language models: Techniques, enhancements, applications, collaboration with llms, and trustworthiness.
\newblock {\em arXiv preprint arXiv:2411.03350}, 2024.

\bibitem{jiang2024survey}
Juyong Jiang, Fan Wang, Jiasi Shen, Sungju Kim, and Sunghun Kim.
\newblock A survey on large language models for code generation.
\newblock {\em arXiv preprint arXiv:2406.00515}, 2024.

\bibitem{nguyen2025codemmlu}
Dung~Manh Nguyen, Thang~Chau Phan, Nam~Le Hai, Tien-Thong Doan, Nam~V. Nguyen, Quang Pham, and Nghi D.~Q. Bui.
\newblock Code{MMLU}: A multi-task benchmark for assessing code understanding \& reasoning capabilities of code{LLM}s.
\newblock In {\em The Thirteenth International Conference on Learning Representations}, 2025.

\bibitem{ye2025rational}
Yining Ye, Xin Cong, Shizuo Tian, Yujia Qin, Chong Liu, Yankai Lin, Zhiyuan Liu, and Maosong Sun.
\newblock Rational decision-making agent with learning internal utility judgment.
\newblock In {\em The Thirteenth International Conference on Learning Representations}, 2025.

\bibitem{lin2025far}
Minhua Lin, Hui Liu, Xianfeng Tang, Jingying Zeng, Zhenwei Dai, Chen Luo, Zheng Li, Xiang Zhang, Qi~He, and Suhang Wang.
\newblock How far are llms from real search? a comprehensive study on efficiency, completeness, and inherent capabilities.
\newblock {\em arXiv preprint arXiv:2502.18387}, 2025.

\bibitem{dai2024safe}
Josef Dai, Xuehai Pan, Ruiyang Sun, Jiaming Ji, Xinbo Xu, Mickel Liu, Yizhou Wang, and Yaodong Yang.
\newblock Safe {RLHF}: Safe reinforcement learning from human feedback.
\newblock In {\em The Twelfth International Conference on Learning Representations}, 2024.

\bibitem{li2025data}
Xiaomin Li, Mingye Gao, Zhiwei Zhang, Jingxuan Fan, and Weiyu Li.
\newblock Data-adaptive safety rules for training reward models.
\newblock {\em arXiv preprint arXiv:2501.15453}, 2025.

\bibitem{shen2023large}
Tianhao Shen, Renren Jin, Yufei Huang, Chuang Liu, Weilong Dong, Zishan Guo, Xinwei Wu, Yan Liu, and Deyi Xiong.
\newblock Large language model alignment: A survey.
\newblock {\em arXiv preprint arXiv:2309.15025}, 2023.

\bibitem{zhang2025catastrophic}
Zhiwei Zhang, Fali Wang, Xiaomin Li, Zongyu Wu, Xianfeng Tang, Hui Liu, Qi~He, Wenpeng Yin, and Suhang Wang.
\newblock Catastrophic failure of {LLM} unlearning via quantization.
\newblock In {\em The Thirteenth International Conference on Learning Representations}, 2025.

\bibitem{bai2022training}
Yuntao Bai, Andy Jones, Kamal Ndousse, Amanda Askell, Anna Chen, Nova DasSarma, Dawn Drain, Stanislav Fort, Deep Ganguli, Tom Henighan, et~al.
\newblock Training a helpful and harmless assistant with reinforcement learning from human feedback.
\newblock {\em arXiv preprint arXiv:2204.05862}, 2022.

\bibitem{ouyang2022training}
Long Ouyang, Jeffrey Wu, Xu~Jiang, Diogo Almeida, Carroll Wainwright, Pamela Mishkin, Chong Zhang, Sandhini Agarwal, Katarina Slama, Alex Ray, et~al.
\newblock Training language models to follow instructions with human feedback.
\newblock {\em Advances in neural information processing systems}, 35:27730--27744, 2022.

\bibitem{schulman2017proximal}
John Schulman, Filip Wolski, Prafulla Dhariwal, Alec Radford, and Oleg Klimov.
\newblock Proximal policy optimization algorithms.
\newblock {\em arXiv preprint arXiv:1707.06347}, 2017.

\bibitem{ahmadian2024back}
Arash Ahmadian, Chris Cremer, Matthias Gall{\'e}, Marzieh Fadaee, Julia Kreutzer, Olivier Pietquin, Ahmet {\"U}st{\"u}n, and Sara Hooker.
\newblock Back to basics: Revisiting reinforce style optimization for learning from human feedback in llms.
\newblock {\em arXiv preprint arXiv:2402.14740}, 2024.

\bibitem{shao2024deepseekmath}
Zhihong Shao, Peiyi Wang, Qihao Zhu, Runxin Xu, Junxiao Song, Xiao Bi, Haowei Zhang, Mingchuan Zhang, YK~Li, Y~Wu, et~al.
\newblock Deepseekmath: Pushing the limits of mathematical reasoning in open language models.
\newblock {\em arXiv preprint arXiv:2402.03300}, 2024.

\bibitem{gao2023scaling}
Leo Gao, John Schulman, and Jacob Hilton.
\newblock Scaling laws for reward model overoptimization.
\newblock In {\em International Conference on Machine Learning}, pages 10835--10866. PMLR, 2023.

\bibitem{yang2024regularizing}
Rui Yang, Ruomeng Ding, Yong Lin, Huan Zhang, and Tong Zhang.
\newblock Regularizing hidden states enables learning generalizable reward model for {LLM}s.
\newblock In {\em The Thirty-eighth Annual Conference on Neural Information Processing Systems}, 2024.

\bibitem{gulcehre2023reinforced}
Caglar Gulcehre, Tom~Le Paine, Srivatsan Srinivasan, Ksenia Konyushkova, Lotte Weerts, Abhishek Sharma, Aditya Siddhant, Alex Ahern, Miaosen Wang, Chenjie Gu, et~al.
\newblock Reinforced self-training (rest) for language modeling.
\newblock {\em arXiv preprint arXiv:2308.08998}, 2023.

\bibitem{dong2023raft}
Hanze Dong, Wei Xiong, Deepanshu Goyal, Yihan Zhang, Winnie Chow, Rui Pan, Shizhe Diao, Jipeng Zhang, KaShun SHUM, and Tong Zhang.
\newblock {RAFT}: Reward ranked finetuning for generative foundation model alignment.
\newblock {\em Transactions on Machine Learning Research}, 2023.

\bibitem{gui2024bonbon}
Lin Gui, Cristina Garbacea, and Victor Veitch.
\newblock Bo{NB}on alignment for large language models and the sweetness of best-of-n sampling.
\newblock In {\em The Thirty-eighth Annual Conference on Neural Information Processing Systems}, 2024.

\bibitem{coste2024reward}
Thomas Coste, Usman Anwar, Robert Kirk, and David Krueger.
\newblock Reward model ensembles help mitigate overoptimization.
\newblock In {\em The Twelfth International Conference on Learning Representations}, 2024.

\bibitem{eisenstein2023helping}
Jacob Eisenstein, Chirag Nagpal, Alekh Agarwal, Ahmad Beirami, Alex D'Amour, DJ~Dvijotham, Adam Fisch, Katherine Heller, Stephen Pfohl, Deepak Ramachandran, et~al.
\newblock Helping or herding? reward model ensembles mitigate but do not eliminate reward hacking.
\newblock {\em arXiv preprint arXiv:2312.09244}, 2023.

\bibitem{zhang2024improving}
Shun Zhang, Zhenfang Chen, Sunli Chen, Yikang Shen, Zhiqing Sun, and Chuang Gan.
\newblock Improving reinforcement learning from human feedback with efficient reward model ensemble.
\newblock {\em arXiv preprint arXiv:2401.16635}, 2024.

\bibitem{yan2024reward}
Yuzi Yan, Xingzhou Lou, Jialian Li, Yiping Zhang, Jian Xie, Chao Yu, Yu~Wang, Dong Yan, and Yuan Shen.
\newblock Reward-robust rlhf in llms.
\newblock {\em arXiv preprint arXiv:2409.15360}, 2024.

\bibitem{rame2024warm}
Alexandre Ram{\'e}, Nino Vieillard, L{\'e}onard Hussenot, Robert Dadashi, Geoffrey Cideron, Olivier Bachem, and Johan Ferret.
\newblock Warm: On the benefits of weight averaged reward models.
\newblock {\em arXiv preprint arXiv:2401.12187}, 2024.

\bibitem{zhang2024mitigating}
Xiaoying Zhang, Jean-Francois Ton, Wei Shen, Hongning Wang, and Yang Liu.
\newblock Mitigating reward overoptimization via lightweight uncertainty estimation.
\newblock In {\em The Thirty-eighth Annual Conference on Neural Information Processing Systems}, 2024.

\bibitem{moskovitz2024confronting}
Ted Moskovitz, Aaditya~K Singh, DJ~Strouse, Tuomas Sandholm, Ruslan Salakhutdinov, Anca Dragan, and Stephen~Marcus McAleer.
\newblock Confronting reward model overoptimization with constrained {RLHF}.
\newblock In {\em The Twelfth International Conference on Learning Representations}, 2024.

\bibitem{liu2024provably}
Zhihan Liu, Miao Lu, Shenao Zhang, Boyi Liu, Hongyi Guo, Yingxiang Yang, Jose Blanchet, and Zhaoran Wang.
\newblock Provably mitigating overoptimization in {RLHF}: Your {SFT} loss is implicitly an adversarial regularizer.
\newblock In {\em The Thirty-eighth Annual Conference on Neural Information Processing Systems}, 2024.

\bibitem{zhang2024overcoming}
Xiaoying Zhang, Jean-Francois Ton, Wei Shen, Hongning Wang, and Yang Liu.
\newblock Overcoming reward overoptimization via adversarial policy optimization with lightweight uncertainty estimation.
\newblock {\em arXiv preprint arXiv:2403.05171}, 2024.

\bibitem{laidlaw2024correlated}
Cassidy Laidlaw, Shivam Singhal, and Anca Dragan.
\newblock Correlated proxies: A new definition and improved mitigation for reward hacking.
\newblock {\em arXiv preprint arXiv:2403.03185}, 2024.

\bibitem{chen2024odin}
Lichang Chen, Chen Zhu, Jiuhai Chen, Davit Soselia, Tianyi Zhou, Tom Goldstein, Heng Huang, Mohammad Shoeybi, and Bryan Catanzaro.
\newblock Odin: Disentangled reward mitigates hacking in rlhf.
\newblock In {\em International Conference on Machine Learning}, pages 7935--7952. PMLR, 2024.

\bibitem{dai2025mitigating}
Juntao Dai, Taiye Chen, Yaodong Yang, Qian Zheng, and Gang Pan.
\newblock Mitigating reward over-optimization in {RLHF} via behavior-supported regularization.
\newblock In {\em The Thirteenth International Conference on Learning Representations}, 2025.

\bibitem{wang2024helpsteer}
Zhilin Wang, Yi~Dong, Olivier Delalleau, Jiaqi Zeng, Gerald Shen, Daniel Egert, Jimmy~J. Zhang, Makesh~Narsimhan Sreedhar, and Oleksii Kuchaiev.
\newblock Helpsteer 2: Open-source dataset for training top-performing reward models.
\newblock In {\em The Thirty-eight Conference on Neural Information Processing Systems Datasets and Benchmarks Track}, 2024.

\bibitem{wang2024interpretable}
Haoxiang Wang, Wei Xiong, Tengyang Xie, Han Zhao, and Tong Zhang.
\newblock Interpretable preferences via multi-objective reward modeling and mixture-of-experts.
\newblock In {\em Findings of the Association for Computational Linguistics: EMNLP 2024}, pages 10582--10592, 2024.

\bibitem{cui2023ultrafeedback}
Ganqu Cui, Lifan Yuan, Ning Ding, Guanming Yao, Wei Zhu, Yuan Ni, Guotong Xie, Zhiyuan Liu, and Maosong Sun.
\newblock Ultrafeedback: Boosting language models with high-quality feedback.
\newblock 2023.

\bibitem{kim2023prometheus}
Seungone Kim, Jamin Shin, Yejin Cho, Joel Jang, Shayne Longpre, Hwaran Lee, Sangdoo Yun, Seongjin Shin, Sungdong Kim, James Thorne, et~al.
\newblock Prometheus: Inducing fine-grained evaluation capability in language models.
\newblock In {\em The Twelfth International Conference on Learning Representations}, 2023.

\bibitem{li2024generation}
Dawei Li, Bohan Jiang, Liangjie Huang, Alimohammad Beigi, Chengshuai Zhao, Zhen Tan, Amrita Bhattacharjee, Yuxuan Jiang, Canyu Chen, Tianhao Wu, et~al.
\newblock From generation to judgment: Opportunities and challenges of llm-as-a-judge.
\newblock {\em arXiv preprint arXiv:2411.16594}, 2024.

\bibitem{gu2024survey}
Jiawei Gu, Xuhui Jiang, Zhichao Shi, Hexiang Tan, Xuehao Zhai, Chengjin Xu, Wei Li, Yinghan Shen, Shengjie Ma, Honghao Liu, et~al.
\newblock A survey on llm-as-a-judge.
\newblock {\em arXiv preprint arXiv:2411.15594}, 2024.

\bibitem{wang2024helpsteer2}
Zhilin Wang, Yi~Dong, Olivier Delalleau, Jiaqi Zeng, Gerald Shen, Daniel Egert, Jimmy~J. Zhang, Makesh~Narsimhan Sreedhar, and Oleksii Kuchaiev.
\newblock Helpsteer2: Open-source dataset for training top-performing reward models, 2024.

\bibitem{lambert2024rewardbench}
Nathan Lambert, Valentina Pyatkin, Jacob Morrison, LJ~Miranda, Bill~Yuchen Lin, Khyathi Chandu, Nouha Dziri, Sachin Kumar, Tom Zick, Yejin Choi, Noah~A. Smith, and Hannaneh Hajishirzi.
\newblock Rewardbench: Evaluating reward models for language modeling, 2024.

\bibitem{liu2025rmbench}
Yantao Liu, Zijun Yao, Rui Min, Yixin Cao, Lei Hou, and Juanzi Li.
\newblock {RM}-bench: Benchmarking reward models of language models with subtlety and style.
\newblock In {\em The Thirteenth International Conference on Learning Representations}, 2025.

\bibitem{bradley1952rank}
Ralph~Allan Bradley and Milton~E Terry.
\newblock Rank analysis of incomplete block designs: I. the method of paired comparisons.
\newblock {\em Biometrika}, 39(3/4):324--345, 1952.

\bibitem{wang2023helpsteer}
Zhilin Wang, Yi~Dong, Jiaqi Zeng, Virginia Adams, Makesh~Narsimhan Sreedhar, Daniel Egert, Olivier Delalleau, Jane~Polak Scowcroft, Neel Kant, Aidan Swope, et~al.
\newblock Helpsteer: Multi-attribute helpfulness dataset for steerlm.
\newblock {\em arXiv preprint arXiv:2311.09528}, 2023.

\bibitem{wang2024arithmetic}
Haoxiang Wang, Yong Lin, Wei Xiong, Rui Yang, Shizhe Diao, Shuang Qiu, Han Zhao, and Tong Zhang.
\newblock Arithmetic control of llms for diverse user preferences: Directional preference alignment with multi-objective rewards.
\newblock In {\em ACL}, 2024.

\bibitem{stiennon2020learning}
Nisan Stiennon, Long Ouyang, Jeffrey Wu, Daniel Ziegler, Ryan Lowe, Chelsea Voss, Alec Radford, Dario Amodei, and Paul~F Christiano.
\newblock Learning to summarize with human feedback.
\newblock {\em Advances in neural information processing systems}, 33:3008--3021, 2020.

\bibitem{wu2024pairwise}
Tianhao Wu, Banghua Zhu, Ruoyu Zhang, Zhaojin Wen, Kannan Ramchandran, and Jiantao Jiao.
\newblock Pairwise proximal policy optimization: Language model alignment with comparative {RL}.
\newblock In {\em First Conference on Language Modeling}, 2024.

\bibitem{fu2025reward}
Jiayi Fu, Xuandong Zhao, Chengyuan Yao, Heng Wang, Qi~Han, and Yanghua Xiao.
\newblock Reward shaping to mitigate reward hacking in rlhf.
\newblock {\em arXiv preprint arXiv:2502.18770}, 2025.

\bibitem{liu2024skywork}
Chris~Yuhao Liu, Liang Zeng, Jiacai Liu, Rui Yan, Jujie He, Chaojie Wang, Shuicheng Yan, Yang Liu, and Yahui Zhou.
\newblock Skywork-reward: Bag of tricks for reward modeling in llms.
\newblock {\em arXiv preprint arXiv:2410.18451}, 2024.

\bibitem{team2024gemma}
Gemma Team, Thomas Mesnard, Cassidy Hardin, Robert Dadashi, Surya Bhupatiraju, Shreya Pathak, Laurent Sifre, Morgane Rivi{\`e}re, Mihir~Sanjay Kale, Juliette Love, et~al.
\newblock Gemma: Open models based on gemini research and technology.
\newblock {\em arXiv preprint arXiv:2403.08295}, 2024.

\bibitem{coste2023reward}
Thomas Coste, Usman Anwar, Robert Kirk, and David Krueger.
\newblock Reward model ensembles help mitigate overoptimization.
\newblock {\em arXiv preprint arXiv:2310.02743}, 2023.

\bibitem{jiang2023mistral7b}
Albert~Q. Jiang, Alexandre Sablayrolles, Arthur Mensch, Chris Bamford, Devendra~Singh Chaplot, Diego de~las Casas, Florian Bressand, Gianna Lengyel, Guillaume Lample, Lucile Saulnier, Lélio~Renard Lavaud, Marie-Anne Lachaux, Pierre Stock, Teven~Le Scao, Thibaut Lavril, Thomas Wang, Timothée Lacroix, and William~El Sayed.
\newblock Mistral 7b, 2023.

\bibitem{touvron2023llama}
Hugo Touvron, Louis Martin, Kevin Stone, Peter Albert, Amjad Almahairi, Yasmine Babaei, Nikolay Bashlykov, Soumya Batra, Prajjwal Bhargava, Shruti Bhosale, et~al.
\newblock Llama 2: Open foundation and fine-tuned chat models.
\newblock {\em arXiv preprint arXiv:2307.09288}, 2023.

\bibitem{wang2024secrets}
Binghai Wang, Rui Zheng, Lu~Chen, Yan Liu, Shihan Dou, Caishuang Huang, Wei Shen, Senjie Jin, Enyu Zhou, Chenyu Shi, et~al.
\newblock Secrets of rlhf in large language models part ii: Reward modeling.
\newblock {\em arXiv preprint arXiv:2401.06080}, 2024.

\bibitem{grattafiori2024llama}
Aaron Grattafiori, Abhimanyu Dubey, Abhinav Jauhri, Abhinav Pandey, Abhishek Kadian, Ahmad Al-Dahle, Aiesha Letman, Akhil Mathur, Alan Schelten, Alex Vaughan, et~al.
\newblock The llama 3 herd of models.
\newblock {\em arXiv preprint arXiv:2407.21783}, 2024.

\bibitem{dong2024rlhf}
Hanze Dong, Wei Xiong, Bo~Pang, Haoxiang Wang, Han Zhao, Yingbo Zhou, Nan Jiang, Doyen Sahoo, Caiming Xiong, and Tong Zhang.
\newblock Rlhf workflow: From reward modeling to online rlhf.
\newblock {\em arXiv preprint arXiv:2405.07863}, 2024.

\bibitem{wen2025rethinking}
Xueru Wen, Jie Lou, Yaojie Lu, Hongyu Lin, XingYu, Xinyu Lu, Ben He, Xianpei Han, Debing Zhang, and Le~Sun.
\newblock Rethinking reward model evaluation: Are we barking up the wrong tree?
\newblock In {\em The Thirteenth International Conference on Learning Representations}, 2025.

\bibitem{razin2025makes}
Noam Razin, Zixuan Wang, Hubert Strauss, Stanley Wei, Jason~D Lee, and Sanjeev Arora.
\newblock What makes a reward model a good teacher? an optimization perspective.
\newblock {\em arXiv preprint arXiv:2503.15477}, 2025.

\bibitem{lin2023mitigating}
Yong Lin, Hangyu Lin, Wei Xiong, Shizhe Diao, Jianmeng Liu, Jipeng Zhang, Rui Pan, Haoxiang Wang, Wenbin Hu, Hanning Zhang, et~al.
\newblock Mitigating the alignment tax of rlhf.
\newblock {\em arXiv preprint arXiv:2309.06256}, 2023.

\bibitem{sun2025rethinking}
Hao Sun, Yunyi Shen, and Jean-Francois Ton.
\newblock Rethinking reward modeling in preference-based large language model alignment.
\newblock In {\em The Thirteenth International Conference on Learning Representations}, 2025.

\bibitem{zhu2023starling}
Banghua Zhu, Evan Frick, Tianhao Wu, Hanlin Zhu, and Jiantao Jiao.
\newblock Starling-7b: Improving llm helpfulness \& harmlessness with rlaif.
\newblock {\em 2023}, 2023.

\bibitem{adler2024nemotron}
Bo~Adler, Niket Agarwal, Ashwath Aithal, Dong~H Anh, Pallab Bhattacharya, Annika Brundyn, Jared Casper, Bryan Catanzaro, Sharon Clay, Jonathan Cohen, et~al.
\newblock Nemotron-4 340b technical report.
\newblock {\em arXiv preprint arXiv:2406.11704}, 2024.

\bibitem{wu2023fine}
Zeqiu Wu, Yushi Hu, Weijia Shi, Nouha Dziri, Alane Suhr, Prithviraj Ammanabrolu, Noah~A Smith, Mari Ostendorf, and Hannaneh Hajishirzi.
\newblock Fine-grained human feedback gives better rewards for language model training.
\newblock {\em Advances in Neural Information Processing Systems}, 36:59008--59033, 2023.

\bibitem{guo2023beyond}
Geyang Guo, Ranchi Zhao, Tianyi Tang, Wayne~Xin Zhao, and Ji-Rong Wen.
\newblock Beyond imitation: Leveraging fine-grained quality signals for alignment.
\newblock {\em arXiv preprint arXiv:2311.04072}, 2023.

\bibitem{chiang2024chatbot}
Wei-Lin Chiang, Lianmin Zheng, Ying Sheng, Anastasios~Nikolas Angelopoulos, Tianle Li, Dacheng Li, Banghua Zhu, Hao Zhang, Michael Jordan, Joseph~E Gonzalez, et~al.
\newblock Chatbot arena: An open platform for evaluating llms by human preference.
\newblock In {\em Forty-first International Conference on Machine Learning}, 2024.

\bibitem{nakano2021webgpt}
Reiichiro Nakano, Jacob Hilton, Suchir Balaji, Jeff Wu, Long Ouyang, Christina Kim, Christopher Hesse, Shantanu Jain, Vineet Kosaraju, William Saunders, et~al.
\newblock Webgpt: Browser-assisted question-answering with human feedback.
\newblock {\em arXiv preprint arXiv:2112.09332}, 2021.

\bibitem{wang2023math}
Peiyi Wang, Lei Li, Zhihong Shao, RX~Xu, Damai Dai, Yifei Li, Deli Chen, Yu~Wu, and Zhifang Sui.
\newblock Math-shepherd: Verify and reinforce llms step-by-step without human annotations.
\newblock {\em arXiv preprint arXiv:2312.08935}, 2023.

\bibitem{luo2024improve}
Liangchen Luo, Yinxiao Liu, Rosanne Liu, Samrat Phatale, Meiqi Guo, Harsh Lara, Yunxuan Li, Lei Shu, Yun Zhu, Lei Meng, et~al.
\newblock Improve mathematical reasoning in language models by automated process supervision.
\newblock {\em arXiv preprint arXiv:2406.06592}, 2024.

\bibitem{zhang2025lessons}
Zhenru Zhang, Chujie Zheng, Yangzhen Wu, Beichen Zhang, Runji Lin, Bowen Yu, Dayiheng Liu, Jingren Zhou, and Junyang Lin.
\newblock The lessons of developing process reward models in mathematical reasoning.
\newblock {\em arXiv preprint arXiv:2501.07301}, 2025.

\bibitem{zhu2024iterative}
Banghua Zhu, Michael~I Jordan, and Jiantao Jiao.
\newblock Iterative data smoothing: Mitigating reward overfitting and overoptimization in rlhf.
\newblock {\em arXiv preprint arXiv:2401.16335}, 2024.

\bibitem{liu2025rrm}
Tianqi Liu, Wei Xiong, Jie Ren, Lichang Chen, Junru Wu, Rishabh Joshi, Yang Gao, Jiaming Shen, Zhen Qin, Tianhe Yu, Daniel Sohn, Anastasia Makarova, Jeremiah~Zhe Liu, Yuan Liu, Bilal Piot, Abe Ittycheriah, Aviral Kumar, and Mohammad Saleh.
\newblock {RRM}: Robust reward model training mitigates reward hacking.
\newblock In {\em The Thirteenth International Conference on Learning Representations}, 2025.

\bibitem{wang2025beyond}
Chaoqi Wang, Zhuokai Zhao, Yibo Jiang, Zhaorun Chen, Chen Zhu, Yuxin Chen, Jiayi Liu, Lizhu Zhang, Xiangjun Fan, Hao Ma, et~al.
\newblock Beyond reward hacking: Causal rewards for large language model alignment.
\newblock {\em arXiv preprint arXiv:2501.09620}, 2025.

\bibitem{miao2024inform}
Yuchun Miao, Sen Zhang, Liang Ding, Rong Bao, Lefei Zhang, and Dacheng Tao.
\newblock Inform: Mitigating reward hacking in rlhf via information-theoretic reward modeling.
\newblock In {\em The Thirty-eighth Annual Conference on Neural Information Processing Systems}, 2024.

\bibitem{bartlett2006convexity}
Peter~L Bartlett, Michael~I Jordan, and Jon~D McAuliffe.
\newblock Convexity, classification, and risk bounds.
\newblock {\em Journal of the American Statistical Association}, 101(473):138--156, 2006.

\bibitem{kay1993fundamentals}
Steven~M Kay.
\newblock {\em Fundamentals of statistical signal processing: estimation theory}.
\newblock Prentice-Hall, Inc., 1993.

\bibitem{khanov2024args}
Maxim Khanov, Jirayu Burapacheep, and Yixuan Li.
\newblock Args: Alignment as reward-guided search.
\newblock {\em arXiv preprint arXiv:2402.01694}, 2024.

\bibitem{wolf2020transformers}
Thomas Wolf, Lysandre Debut, Victor Sanh, Julien Chaumond, Clement Delangue, Anthony Moi, Pierric Cistac, Tim Rault, R{\'e}mi Louf, Morgan Funtowicz, et~al.
\newblock Transformers: State-of-the-art natural language processing.
\newblock In {\em Proceedings of the 2020 conference on empirical methods in natural language processing: system demonstrations}, pages 38--45, 2020.

\end{thebibliography}
}

\newpage
\appendix
\section{Related Work and Key Differences from Prior Approaches}
\label{appendix:related}
In this section, we provide additional related work and clarify how SMORM differs from existing approaches.

\noindent\textbf{Reward Modeling.} Reward models are designed to produce preference-aligned signals that guide the behavior of language models. In the context of large language models (LLMs), they act as proxies for human preferences, providing feedback to the policy model during the alignment process \cite{ouyang2022training, bai2022training, dong2024rlhf, wen2025rethinking, dong2024rlhf, razin2025makes, lin2023mitigating, sun2025rethinking}.
These models are typically constructed by attaching a classification head to a pretrained LLM, allowing it to assign scores to responses conditioned on prompts \cite{zhu2023starling, adler2024nemotron}. To ensure alignment with principles such as helpfulness, harmlessness, and honesty, reward models are fine-tuned using human preference datasets \cite{wu2023fine, guo2023beyond, dai2024safe, chiang2024chatbot, nakano2021webgpt}. The resulting reward signals are then used in policy optimization, enhancing LLM performance on complex downstream tasks, including mathematical reasoning \cite{shao2024deepseekmath, wang2023math, luo2024improve, zhang2025lessons}.

\noindent\textbf{Mitigating Reward Hacking in RLHF.}
Reward hacking in reinforcement learning from human feedback (RLHF) arises when the policy model exploits imperfections in the reward model, thereby failing to learn the intended behaviors. A variety of strategies have been proposed to address this issue. One line of work focuses on enhancing the reward function using ensemble methods \cite{coste2024reward, eisenstein2023helping, zhang2024improving, yan2024reward, rame2024warm, zhang2024mitigating}. While effective, these approaches typically require training multiple reward models, making them computationally expensive and less feasible for real-world deployment.
Another approach investigates constrained policy optimization \cite{moskovitz2024confronting, zhang2024mitigating, liu2024provably, zhang2024overcoming, laidlaw2024correlated, zhu2024iterative}. However, these methods often suffer from performance instability due to their sensitivity to hyperparameter tuning.
More recently, GRM \cite{yang2024regularizing, dai2025mitigating} incorporates text generation regularization into reward modeling and achieves superior performance compared to prior methods. Nonetheless, the inherent conflict between reward modeling and generation objectives introduces training instability and increases sensitivity to the choice of balancing weights.

ODIN \cite{chen2024odin} is a framework that trains two reward functions for response quality and length, sharing a common embedding space. However, our work differs substantially from ODIN in several key aspects:
\textbf{(1)} While ODIN uses the BT loss for both heads, we theoretically establish a connection between BT loss and multi-attribute regression loss. This advancement enables the integration of multiple fine-grained attributes beyond just response length.
\textbf{(2)} Our work explicitly investigates policy optimization with PPO in out-of-distribution (OOD) settings, a scenario largely overlooked by existing studies. Our empirical results in Sec.~\ref{hacking} also demonstrate that ODIN fails in this setting.
\textbf{(3)} Whereas ODIN focuses solely on mitigating reward hacking, we additionally uncover a complementary benefit: training a single-objective reward function significantly enhances the scoring capability of the multi-objective reward model.

RRM \cite{liu2025rrm} is a general framework that employs a causal approach to learn preferences independent of spurious artifacts via data augmentation during reward model training. However, \textbf{(1)} the data augmentation strategy significantly increases the computational cost of training a reward model; \textbf{(2)} it remains unclear how to adapt such augmentation methods to multi-objective reward modeling, especially when labels consist of fine-grained scores; and \textbf{(3)} its effectiveness in mitigating reward hacking in RLHF settings has yet to be empirically validated. Similarly, \cite{wang2025beyond} propose a causal reward modeling approach that incorporates causal inference to reduce spurious correlations. InfoRM \cite{miao2024inform} introduces a variational information bottleneck objective to filter out irrelevant information during reward modeling. However, similar to GRM, this objective is fundamentally at odds with the goal of accurate reward modeling, and the performance of InfoRM on RewardBench has not been evaluated. In contrast, beyond addressing reward hacking, a key contribution of our work is enhancing the performance of weak multi-objective reward models without requiring additional preference data.


\section{Additional Experiments}
\label{appendix:additionalexps}
\subsection{Preliminary Experiments on Improving both Single- and Multi-Objective Head}
\noindent\textbf{Experimental Setup.} 
We follow the same training setup for SMORM as described in Section~\ref{hacking}. We compare SMORM with the baseline SORM and MORM, using \texttt{Llama-3.2-3B-Instruct}\footnote{\href{https://huggingface.co/meta-llama/Llama-3.2-3B-Instruct}{meta-llama/Llama-3.2-3B-Instruct}} as the backbone model for all methods. The comparison results on RewardBench against SORM and MORM are presented in Table~\ref{pre-sorm} and Table~\ref{pre-morm}, respectively. The corresponding datasets used for training each reward model are also listed in the tables. Note that SORM is trained solely on $\mathcal{D}_S$, MORM on $\mathcal{D}_M$, while SMORM is jointly trained on both $\mathcal{D}_S$ and $\mathcal{D}_M$.

\begin{table}[htbp]
    \centering
    \caption{Results on RewardBench compared to the baseline SORM.}
    \resizebox{\textwidth}{!}{%
    \begin{tabular}{c|lcccccl}
        \toprule
        Dataset \(\mathcal{D}_S\) / \(\mathcal{D}_M\)& Model & Chat & Chat Hard & Safety & Reasoning & RewardBench   \\
        \midrule
        \multirow{3}{*}{\shortstack{UltraFeedback (binarized) /\\UltraFeedback}} 
        &Baseline & 89.1 & 40.7 & 45.2 & 36.7 & 52.9 \\
        &SMORM-M & 91.3 & 39.0 & 50.1 & 41.9 & 55.6  \\
        &SMORM-F & 88.4 & 43.4 & 44.3 & 49.6 & 56.4  \\
        \midrule
        \multirow{3}{*}{\shortstack{Skywork80K/\\HelpSteer2}}
        &Baseline & 73.4 & 60.5 & 79.8 & 49.6 & 65.8 \\
        &SMORM-M & 83.5 & 55.5 & 74.5 & 53.6 & 66.8 \\
        &SMORM-F & 80.4 & 62.1 & 80.7 & 55.1 & 69.6 \\
        \bottomrule
    \end{tabular}%
    }
    \label{pre-sorm}
\end{table}
\begin{table}[htbp]
    \centering
    \caption{Results on RewardBench and RM-Bench compared to the baseline MORM.}
    \resizebox{\textwidth}{!}{%
    \begin{tabular}{c|lcccccc}
        \toprule
        Dataset \(\mathcal{D}_S\) / \(\mathcal{D}_M\)& Model & Chat & Chat Hard & Safety & Reasoning & RewardBench & RM-Bench  \\
        \midrule
        \multirow{4}{*}{\shortstack{HelpSteer2 (binarized)/\\HelpSteer2}}
        &Baseline & 55.8 & 50.4 & 44.8 & 54.2 & 51.3 & 49.2  \\
        &SMORM-M & 48.9 & 48.5 & 51.4 & 75.7 & 56.1 & 50.1  \\
        &SMORM-F & 50.3 & 48.7 & 54.7 & 73.6 & 56.9 & 50.9  \\
        &SMORM-L & 50.7 & 49.3 & 52.9 & 73.3 & 56.8 & 53.0  \\
        \midrule
        \multirow{4}{*}{\shortstack{UltraFeedback (binarized)/\\UltraFeedback}} & Baseline & 70.3 & 46.1 & 41.6 & 42.0 & 50.0 & 50.2 \\
        &SMORM-M & 91.3 & 39.0 & 50.1 & 41.9 & 55.6 & 51.0  \\
        &SMORM-F & 88.4 & 43.4 & 44.3 & 49.6 & 56.4 & 49.8  \\
        &SMORM-L & 90.2 & 40.1 & 54.2 & 40.8 & 56.3 & 54.4  \\
        \bottomrule
    \end{tabular}%
    }
    \label{pre-morm}
\end{table}

\noindent\textbf{Results Analysis.} 
From Table \ref{pre-sorm}, we observe that SMORM-F achieves a higher average score on RewardBench compared to the baseline model. Notably, regardless of whether the single-objective reward model is trained on the same dataset or a different one, SMORM-F consistently outperforms the baseline. This result highlights the flexibility of our SMORM framework.
From Table \ref{pre-morm}, we further observe that by simply aggregating the multi-attribute scores to create a single-objective preference dataset and training a SMORM, SMORM-L achieves average scores \textbf{5.5} and \textbf{6.3} points higher than training a MORM alone. This highlights how easily SMORM can enhance the performance of the multi-objective reward function.

\subsection{Additional Results on Improving Single-Objective Head}
In this section, we provide additional experimental results comparing SMORM-F to baseline single-objective reward functions. We follow the same experimental setup as in Section~\ref{experiments}. Results using Mistral-7B-Instruct as the base model are reported in Table~\ref{tab:sorm_combinedappendix}.
\label{appendix:additionalsingle}
\begin{table*}[h]
\centering
\caption{Comparison of SMORM-F and baselines on RewardBench. Baselines results from \cite{yang2024regularizing}.}
\resizebox{\textwidth}{!}{%
\begin{tabular}{l|ccccc|ccccc}
\toprule
\multirow{2}{*}{\textbf{Reward model}} & \multicolumn{5}{c|}{$\mathcal{D_S}/\mathcal{D}_M$: UnifiedFeedback 400k/UltraFeedback} & \multicolumn{5}{c}{$\mathcal{D_S}/\mathcal{D}_M$: UnifiedFeedback 40k/HelpSteer2} \\
\cmidrule{2-11}
& \textbf{Chat} & \textbf{Chat-Hard} & \textbf{Safety} & \textbf{Reasoning} & \textbf{Avg} & \textbf{Chat} & \textbf{Chat-Hard} & \textbf{Safety} & \textbf{Reasoning} & \textbf{Avg} \\
\midrule
\multicolumn{11}{c}{Base Model: Mistral 7b Instruct} \\
\midrule
Baseline (Single)         
& 96.6 & 52.4 & 86.7 & 69.5 & 76.3
& 94.9 & 51.7 & 64.9 & 62.4 & 68.5 \\
Baseline + margin                 
& 96.4 & 51.5 & 85.3 & 64.8 & 74.5
& 89.7 & 47.1 & 70.7 & 43.6 & 62.8 \\
Label smooth           
& 97.2 & 49.8 & 85.8 & 72.3 & 76.3
& 94.1 & 47.1 & 67.5 & 79.7 & 72.1 \\
Ensemble               
& 96.6 & 51.8 & 85.1 & 73.0 & 76.6
& 89.6 & 50.2 & 72.7 & 59.0 & 69.3 \\
GRM (linear) w/ dpo  
& 98.0 & 53.3 & 86.4 & 75.3 & 78.3
& 95.1 & 47.5 & 82.2 & 74.7 & 74.9 \\
GRM (linear) w/ sft          
& 97.8 & 54.6 & 86.3 & 79.2 & 79.5
& 93.4 & 51.9 & 80.7 & 78.8 & 76.2 \\
GRM w/ dpo          
& 97.8 & 54.0 & 85.7 & 74.4 & 78.0
& 97.8 & 52.4 & 78.0 & 77.3 & 76.4 \\
GRM w/ sft                   
& 98.0 & 55.3 & 85.8 & 71.2 & 77.6
& 94.1 & 48.5 & 83.4 & 77.4 & 75.9 \\
\rowcolor{gray!20}
SMORM-F  
& 97.8 & 55.3 & 85.9 & 80.1 & \textbf{79.8}
& 95.8 & 60.1 & 80.5 & 74.7 & \textbf{77.8} \\
\bottomrule
\end{tabular}
}
\label{tab:sorm_combinedappendix}
\end{table*}

\section{Theory}
\label{appendix:theory}
\subsection{Definitions and Assumptions}\label{subsec:definitions}
Let $K$ be the number of attributes considered. We define $f_\theta$ as the backbone of our reward model, which maps an input question--response pair to a hidden representation of dimension $d$. For each attribute head $k$, the reward score is given by $r_k = \mathbf{w}_k^\top f_\theta$. In particular, $r_s$ (or $r_0$) denotes the score from the single-objective head used to model overall preference (e.g., chosen vs. rejected), while $r_k$ for $k > 0$ corresponds to the outputs of fine-grained attribute-specific heads.

\textbf{Positive-definite covariances.} 
Let
\[
f_c = f_\theta(x_s, y_c), \quad
f_r = f_\theta(x_s, y_r), \quad
f_m = f_\theta(x_m, y_m).
\]
We define the covariance matrices
\[
\Sigma_S := \mathbb{E}_{\mathcal{D}_S}\left[(f_c - f_r)(f_c - f_r)^\top\right]
\quad \text{and} \quad
\Sigma_M := \mathbb{E}_{\mathcal{D}_M}\left[f_m f_m^\top\right],
\]
and assume both are positive definite (PD).

Modern feature extractors (e.g., transformer-based backbones) typically embed inputs into a high-dimensional space of size~$d$, which exceeds the intrinsic rank of the data. Provided that the samples in $\mathcal{D}_S$ (or $\mathcal{D}_M$) do not all lie in a strict lower-dimensional hyperplane, the corresponding empirical covariance matrices will be full-rank and therefore positive definite.

\textbf{Correlation between heads.} 
We naturally assume a positive correlation between the aggregated fine-grained attribute scores and the preference along the chosen/rejected dimension.
Specifically, 
let
\[
\mu_S := \mathbb{E}_{(x_s, y_c, y_r) \sim \mathcal{D}_S} \left[f_\theta(x_s, y_c) - f_\theta(x_s, y_r)\right],
\]
and define
\[
C_M := \mathbb{E}_{(x_m, y_m, \mathbf{r}) \sim \mathcal{D}_M} \left[ f_\theta(x_m, y_m)\, \mathbf{r}^\top \right] \in \mathbb{R}^{d \times K}.
\]
Note that we denote $r$ as the reward function, and $\mathbf{r}$ as the vector of multi-attribute scores. Then, we define
\[
\alpha := \mu_S^\top \Sigma_M^{-1} C_M,
\]
and assume that the sum of its components is non-negative, i.e.,
\[
\mathbf{1}^\top \alpha \ge 0.
\]


Assuming $\E[\mathbf{r}]=\E[\mathbf{w}_S^{\top}]=0$, we define the \emph{raw coupling vector}
\[
  \alpha \;=\;\mu_S^{\top}\Sigma_M^{-1}C_M
           \;=\;
           \Bigl[
             \Cov(\mathbf{r}_1,\;\mathbf{w}_S^\top f),
             \dots,
             \Cov(\mathbf{r}_K,\;\mathbf{w}_S^\top f)
           \Bigr]\;\in\;\R^K,
\]
and normalize
\[
  \beta \;=\;\frac{\alpha}{\|\tilde\mu_S\|^2}
           \;=\;
           \frac{\mu_S^{\top}\Sigma_M^{-1}C_M}
                {\mu_S^{\top}\Sigma_S^{-1}\mu_S}
           \;\in\;\R^K.
\]
Then, $\beta_i > 0$ if the score on attribute $i$ tends to increase as the single-objective score $\mathbf{w}_S^\top f$ increases.

The sum \( \mathbf{1}^\top \alpha \) represents the total (signed) covariance between the single-objective preference score and the aggregate multi-attribute score. In most real-world annotation settings, chosen--rejected labels are used as proxies for overall answer quality. Since higher-quality responses tend to improve multiple fine-grained attributes (e.g., helpfulness, correctness, coherence, etc.), the covariances across these attributes typically sum to a positive total.


\textbf{Hidden Ground Truth.} For each head $r_k$, we denote $r_k^\ast$ as the hidden ground-truth reward function, and we model the labeled score for head $k$ as $g_k = r_k^\ast + \epsilon_k$ for $\varepsilon \sim \calN(0,\Sigma),
\Sigma_{0k}>0$. Here the diagonal entries of $\Sigma$ are $\sigma_{kk}=\operatorname{Var}(\varepsilon_k)$, the noise variance for head $k$.

\subsection{Proof of Theorem \ref{prop:implicit_multi_effect_detailed}}
\label{prooftheorem1}
\newtheorem*{retheorem2}{Theorem \ref{prop:implicit_multi_effect_detailed}}
\begin{retheorem2}[Implicit Multi-Attribute Effect]
Let a reward model be trained under the SMORM framework, and suppose the following conditions hold:
(1) Bounded features: There exists $B<\infty$ such that $\|f_\theta(x,y)\|\le B$ for every $(x,y)$.
(2)
Positive-definite covariances: let
$
  f_c=f_\theta(x_s,y_c),\;
  f_r=f_\theta(x_s,y_r),\;
  f_m =f_\theta(x_m,y_m).$
        \(
          \Sigma_S:=\E_{\mathcal{D}_S}[(f_c-f_r)(f_c-f_r)^\top]
          \;
          \text{and}
          \;
          \Sigma_M:=\E_{\mathcal{D}_M}[f_m\,f_m^\top]
        \)
        are positive-definite matrices.
(3) Positive correlation: Let
        $\mu_S:=\E_{(x_s,y_c,y_r)\sim\mathcal{D}_S}[f_\theta(x_s,y_c)-f_\theta(x_s,y_r)]$
        and let  
        $C_M:=\E_{(x_m,y_m,r)\sim\mathcal{D}_M}[f_\theta(x_m,y_m)\,\mathbf{r}^\top]\in\R^{d\times K}$.
        Then $\alpha:=\mu_S^\top\!\Sigma_M^{-1}C_M\ \text{has non-negative sum, i.e.}\;
            \mathbf 1^\top\alpha \;\ge\;0.$
As the optimization of both reward heads converge to their population minimizers, there exist constants $c=\frac{\mathbf{1}^\top\alpha}{K\,\bigl(\mu_S^\top\,\Sigma_S^{-1}\,\mu_S\bigr)}$ and $\varepsilon\ge0$—depending only
on $B$ and second-order moments—such that for every pair \((x,y)\):
\begin{equation}
\resizebox{0.83\linewidth}{!}{$
r_m(x,y)\;=\;\frac1K\sum_{i=1}^K w_{M,i}^\top f_\theta(x,y)
  \;\;\ge\;\;
  c\,\bigl(w_S^\top f_\theta(x,y)\bigr)\;-\;\varepsilon=cr_s(x,y)\;-\;\varepsilon.
  $}
\end{equation}
\end{retheorem2}

\begin{proof}
Replacing the logistic (BT) and regression losses by squared losses
does not alter \emph{directions} of the minimisers because any strictly
convex proper surrogate has the same first-order optimality conditions
up to a positive scalar factor \cite{bartlett2006convexity}. We therefore analyse the following
least-squares problems:
\[
  \min_{\mathbf{w}_S}\E_S\!\bigl[(\mathbf{w}_S^\top(f_c-f_r)-1)^2\bigr],\qquad
  \min_{\mathbf{w}_M}\E_M\!\bigl[\|\mathbf{w}_M^\top f_m-\mathbf{r}\|_2^2\bigr],
\]
For a generic least-squares objective
\(
  \min_\mathbf{w}\E[(\mathbf{w}^\top u-t)^2],
\)
setting the gradient to zero yields
\(
  \E[u\,u^\top]\,\mathbf{w}=\E[u\,t].
\)
Applying this template we obtain the \emph{population} solutions:
\begin{align}
  \mathbf{w}_S &= \Sigma_S^{-1}\mu_S,
  & 
  \mathbf{w}_M &= \Sigma_M^{-1}C_M.
  \label{eq:normal_solutions}
\end{align}
Define the whitening operator
\(
  \Phi :=\Sigma_S^{-1/2}f\in\R^d,
\)
$\tilde\mu_S :=\Sigma_S^{-1/2}\mu_S,$ and $\tilde C_M :=\Sigma_S^{1/2}\Sigma_M^{-1}C_M$.
Then $\Phi_c=\Sigma_S^{-1/2}f_c, \Phi_r=\Sigma_S^{-1/2}f_r$ and \(\E_S[(\Phi_c-\Phi_r)(\Phi_c-\Phi_r)^\top]=I_d\) and $\tilde\mu_S\neq 0$.
Then \Eqref{eq:normal_solutions} becomes:
\begin{align}
  \mathbf{w}_S &= \Sigma_S^{-1/2}\,\tilde\mu_S,
  &
  \mathbf{w}_M := \Sigma_S^{-1/2}\tilde C_M
.         \label{eq:whitened_heads}
\end{align}
Because \(\tilde\mu_S\neq0\), we write the Euclidean projection of each column of \(\tilde C_M\) onto
\(\tilde\mu_S\):
\begin{equation}
  \tilde C_M=\underbrace{
    \tilde\mu_S\,
    \frac{\tilde\mu_S^\top\tilde C_M}{\|\tilde\mu_S\|^2}
  }_{\text{aligned component}}
  \;+\;
  E,
  \quad
  E^\top\tilde\mu_S = 0.             \label{eq:projection}
\end{equation}
Define the \emph{coupling vector}
\(
  \beta
  :=\tfrac{\tilde\mu_S^\top\tilde C_M}{\|\tilde\mu_S\|^2}=\frac{\alpha}{\|\tilde\mu_S\|^2}\in\R^K.
\)
By assumption (3) we have
\(
  \mathbf1^\top\alpha\ge 0 \,\text{and thus}\;\mathbf1^\top \beta>0.
\)

Then for feature vector \(f\), we have:
\begin{align*}
  \mathbf{w}_M^\top f
  &\stackrel{\Eqref{eq:whitened_heads}}{=}
    \tilde C_M^\top\,\underbrace{\Sigma_S^{-1/2}f}_{\Phi}
    \\
  &\stackrel{\Eqref{eq:projection}}{=}
    \beta\,(\tilde\mu_S^\top\Phi)\;+\;E^\top\Phi
    \\
  &= \beta\,(\mathbf{w}_S^\top f)\;+\;E^\top\Phi.
\end{align*}
Decomposing $\Phi$ by:
\[
  \Phi
  \;=\;
  \underbrace{\frac{\tilde\mu_S^\top\Phi}{\|\tilde\mu_S\|^2}\,\tilde\mu_S}_{\parallel\tilde\mu_S}
  \;+\;
  \underbrace{z}_{\perp\tilde\mu_S},
  \quad
  \tilde\mu_S^\top z = 0,
\]
so that \(E^\top\Phi = E^\top z\).  By bounded‐features,
\(\|f\|\le B\) and hence
\(\|\Phi\|\le B/\sqrt{\lambda_{\min}(\Sigma_S)}\), where $\sqrt{\lambda_{\min}(\Sigma_S)}$ is the square root of the smallest eigenvalue of the .  The projection term
has norm
\(\bigl|\tilde\mu_S^\top\Phi\bigr|/\|\tilde\mu_S\|\le|\mathbf w_S^\top f|\),
so the orthogonal part satisfies
\(\|z\|\le B/\sqrt{\lambda_{\min}(\Sigma_S)}\).  Therefore
\[
  |E^\top z|
  \;\le\;
  \|E\|_{\mathrm{op}}\;\|z\|
  \;\le\;
  \underbrace{\frac{B}{\sqrt{\lambda_{\min}(\Sigma_S)}}\,\|E\|_{\mathrm{op}}}_{\varepsilon}.
\]
Averaging over the \(K\) attributes and using
\(
  c:=\max\{0,\mathbf1^\top\beta\}/K
\)
gives
\begin{equation}
  \frac1K\!\sum_{i=1}^K \mathbf{w}_{M,i}^\top f
  \;\ge\;
  c\,\mathbf{w}_S^\top f
  \;-\;
  \frac{\varepsilon}{K}.
  \label{eq:coupling_plus_error}
\end{equation}
where $c
    = \frac{\mathbf1^\top\beta}{K}
    =\frac{1^\top\alpha}{K\|\tilde\mu_S\|^2}.$


\medskip
\noindent\textbf{Monotone lower bound.}
Because $\mathbf1^\top\beta>0$ by assumption~(3), we have $c>0$.
Define the linear function $L(s)\coloneqq c\,s-\varepsilon/K$.
Inequality~\eqref{eq:coupling_plus_error} reads
\[
  r_M(x,y)\;\;=\;\;\frac1K\sum_{i=1}^K\mathbf w_{M,i}^{\top}f_\theta(x,y)
  \;\;\ge\;\;L\!\bigl(r_S(x,y)\bigr),
\]
so $r_M$ is bounded from below by the \emph{increasing} map $L(\,\cdot\,)$
of the single-objective score $r_S$.  Consequently, for any two responses
$(x_1,y_1)$ and $(x_2,y_2)$,
\[
  r_S(x_1,y_1)\;\ge\;r_S(x_2,y_2)
  \quad\Longrightarrow\quad
  L\!\bigl(r_S(x_1,y_1)\bigr)\;\ge\;L\!\bigl(r_S(x_2,y_2)\bigr),
\]
and hence the lower bound on the multi-attribute average is larger (or
equal) whenever the single-objective score is larger.  In particular,
for any threshold $\tau$,
\[
  r_S(x,y)\;\ge\;\tau
  \quad\Longrightarrow\quad
  r_M(x,y)\;\ge\;c\,\tau-\frac{\varepsilon}{K}.
\]
This completes the proof of
Theorem~\ref{prop:implicit_multi_effect_detailed}.

\end{proof}

\textbf{Pure-SORM failure.}\;
If no multi-attribute head is trained (\(C_M=0\)),
then \(\alpha=0\Rightarrow c=0\), so the lower bound degenerates
to \(\!r_M\ge -\varepsilon\), offering no positive coupling between
single-objective and multi-attribute scores.

\subsection{Proof of Lemma \ref{lemma1}.}
\newtheorem*{relemma}{Lemma \ref{lemma1}} 
\begin{relemma}
Let \( y_A, y_B \) be a pair of responses. Assume \( g_s(y) \) is the ground truth score and \( r_s(y) \) is the predicted score under a Bradley–Terry model. Then:
\[
\P(y_A \succ y_B) = \sigma\big(r_s(y_A) - r_s(y_B)\big), \quad 
\P^\star(y_A \succ y_B) = \sigma\big(g_s(y_A) - g_s(y_B)\big),
\]
where \( \sigma(t) = \frac{1}{1 + e^{-t}} \). The expected preference error satisfies:
\[
\E_{\mathcal{D}_S} \left| \P(y_A \succ y_B) - \P^\star(y_A \succ y_B) \right|
\le 
\frac{1}{4} \E_{\mathcal{D}_S} \left( \sqrt{2\, \text{MSE}(r_s)} \right),
\]
with \(\text{MSE}(r_s) = \big(r_s(y) - g_s(y)\big)^2\).
Similarly, for a multi-objective reward model with predicted score \( r_m \) and ground truth \( g_m \), let:
$
e_m = r_m(y_A) - r_m(y_B), \quad e_m^\star = g_m(y_A) - g_m(y_B),
$
then the error is bounded as:
\[
\E_{\mathcal{D}_M} \left| e_m - e_m^\star \right| \le \E_{\mathcal{D}_M}\left(\sqrt{2\, \text{MSE}(r_m)}\right).
\]
\end{relemma}
\begin{proof}
For a pair of responses \( y_A \) and \( y_B \), the Bradley--Terry model defines the probability that \( y_A \) is preferred (i.e., has a higher overall reward) as:
\[
\P(y_A \succ y_B) = \sigma\bigl(r_s(y_A) - r_s(y_B)\bigr),
\]
where \( \sigma(t) = \frac{1}{1 + e^{-t}} \) is the sigmoid function, and \( r_s(v) \) is the model’s predicted overall score. The corresponding ground-truth preference probability, based on labeled scores, is given by:
\[
\P^\star(y_A \succ y_B) = \sigma\bigl(g_s(y_A) - g_s(y_B)\bigr),
\]
where \( g_s(\cdot) \) denotes the ground-truth reward.

The prediction error in probability space is the absolute difference:
\[
\Delta_{AB} = \left| \sigma\bigl(r_s(y_A) - r_s(y_B)\bigr) - \sigma\bigl(g_s(y_A) - g_s(y_B)\bigr) \right|.
\]

Since the sigmoid derivative satisfies
\[
\sigma'(t) = \sigma(t)(1 - \sigma(t)),
\]
and reaches its maximum value of \( \frac{1}{4} \) at \( t = 0 \), we have \( \sigma'(t) \le \frac{1}{4} \) for all \( t \). This implies that the sigmoid function is \( \frac{1}{4} \)-Lipschitz:
\[
|\sigma(a) - \sigma(b)| \le \frac{1}{4} |a - b|, \quad \forall\, a,b \in \mathbb{R}.
\]

Applying this to our setup, we obtain:
\[
\left| \P(y_A \succ y_B) - \P^\star(y_A \succ y_B) \right| \le \frac{1}{4} \left| \left( r_s(y_A) - r_s(y_B) \right) - \left( g_s(y_A) - g_s(y_B) \right) \right|.
\]

Taking expectation and applying the Cauchy--Schwarz inequality, we can further bound the expected pairwise error by:
\[
\text{Pairwise‑error} \le \frac{1}{4} \sqrt{2\, \text{MSE}(r_s)},
\]
where \( \text{MSE}(r_s) := \mathbb{E}_y\left[(r_s(y) - g_s(y))^2\right] \) is the mean squared error of the predicted scores.

This result shows that minimizing the pointwise MSE of the reward model also reduces the upper bound on the pairwise misordering error, thereby improving preference consistency.

Similarly, for a multi-objective reward model with predicted score \( r_m \) and ground truth \( g_m \), let:
$
e_m = r_m(v_A) - r_m(v_B), \quad e_m^\star = g_m(v_A) - g_m(v_B),
$
then the error is bounded as:
\[
\E_{\mathcal{D}_M} \left| e_m - e_m^\star \right| \le \E_{\mathcal{D}_M}\left(\sqrt{2\, \text{MSE}(r_m)}\right).
\]

\end{proof}
To relate this to the Bradley--Terry loss, recall that the Bradley--Terry loss for a pair \( (y_A, y_B) \), where \( y_A \) is the preferred response, is given by:
\[
\ell_{\mathrm{BT}} = -\log \sigma\bigl(r_s(y_A) - r_s(y_B)\bigr).
\]
If the model’s predicted scores \( r_s(v) \) are close to the ground-truth scores \( g_s(v) \)—i.e., the mean squared error (MSE) is small—then the difference \( r_s(y_A) - r_s(y_B) \) will closely approximate \( g_s(y_A) - g_s(y_B) \). By the Lipschitz continuity of the sigmoid function, this implies that the predicted probability under the model and the ideal ground-truth probability will also be close. Consequently, the pairwise preference error will be small.

This reasoning provides a theoretical justification that minimizing the MSE of individual predictions naturally leads to accurate pairwise probability estimates, as evaluated by the Bradley--Terry loss.

Furthermore, while the single-objective head \( r_s \) is trained using the Bradley--Terry likelihood, our previous bound shows that the expected Bradley--Terry test risk,
\[
\mathcal{R}_{\mathrm{BT}} = \mathbb{E}_{(y_A, y_B)}\left[-\log \sigma\bigl(r_s(y_A) - r_s(y_B)\bigr)\right],
\]
is a 1-Lipschitz function of the score difference \( r_s(y_A) - r_s(y_B) \). Therefore, controlling the variance of the individual scores—captured by \( \text{MSE}_0 = \mathbb{E}[(r_s(y) - g_s(y))^2] \)—directly bounds the generalization error under the Bradley--Terry loss, up to a constant factor.

\subsection{Proof of Theorem \ref{theorem2}}
\newtheorem*{retheorem}{Theorem \ref{theorem2}}
\begin{retheorem}
Under the same assumptions as in Theorem~\ref{prop:implicit_multi_effect_detailed} and assuming that the feature extractor $f_\theta$ is differentiable, let \(\widehat{\theta}\) denote the maximum likelihood estimator (MLE) of the ground truth optimal parameter \(\theta^\star\). Let \(\widehat{\theta}_{\text{s}}\) and \(\widehat{\theta}_{\text{m}}\) denote the maximum likelihood estimators of the single- and multi-objective reward functions, respectively.
Define
\(M_S(y) = \mathbf{w}_S^\top f_{\theta^\star}(y),\) \(M_M(y) = \mathbf{w}_M^\top f_{\theta^\star}(y)\). Then, for a response \(y\), the mean squared error (MSE) of the predicted reward can be approximated as:
\[ \label{eq:thm-MSE0}
\resizebox{0.99\linewidth}{!}{$
\operatorname{MSE}_S \approx \nabla_\theta M_S(y)^{\top} \operatorname{Cov}\left( \widehat{\theta}_s \right) \nabla_\theta M_S(y) + \sigma_{00}, \operatorname{MSE}_M \approx \nabla_\theta M_M(y)^{\top} \operatorname{Cov}\left( \widehat{\theta}_m \right) \nabla_\theta M_M(y) + \sigma_{00},$}
\]
where $\sigma_{00}$ is the intrinsic randomness in the label. Moreover, SMORM yields lower asymptotic MSE for both the single- and multi-objective heads compared to training either head alone:
\begin{equation} \label{eq:thm-MSE0-inequality}
\operatorname{MSE}_S^{\text{SMORM}} 
<
\operatorname{MSE}_S^{\text{single}},\quad
\operatorname{MSE}_M^{\text{SMORM}} 
<
\operatorname{MSE}_M^{\text{multi}}
\end{equation}
\end{retheorem}

\begin{proof}
\textbf{Fisher matrix.}
The Fisher information is a way of measuring the amount of information that an observable random variable carries about an unknown parameter. Mathematically, for a parameter vector $\theta$ 
and data $D$ with likelihood $p(D \mid \theta)$, the Fisher information is
\begin{equation}\label{eq:Fisher}
    \mathcal{I}(\theta) \bydef 
\E_{D \sim p(\cdot \mid \theta)}
\left[\,\nabla_\theta \log p(D \mid \theta)
\nabla_\theta \log p(D \mid \theta)^{\top}\right].
\end{equation}

Intuitively, it measures how sensitive the log-likelihood is to small changes in $\theta$. More curvature means larger $\mathcal{I}(\theta)$ and thus implies that we can estimate $\theta$ more precisely. The celebrated Cramér–Rao bound says that (under mild conditions) any unbiased estimator's covariance is at least $\left[\mathcal{I}(\theta)\right]^{-1}$ \citep{kay1993fundamentals}.

In our square-loss, Gaussian-noise setting, the empirical Fisher matrix becomes the empirical sum of outer products of gradients: 
\[
\mathcal I^{(\text{regime})}(\theta)
= \frac1{n}\sum_{i=1}^{n}
\sum_{k\in\calK_{\text{train}}}
\frac{1}{\sigma_{kk}}
\bigl[\nabla_\theta r_k(y_i)\bigr]
\bigl[\nabla_\theta r_k(y_i)\bigr]^{\top}.
\]
For single‑head reward model that evaluates the overall quality, we have $\calK=\{0\}$. For $K$‑attribute reward model that evaluates the response according to $K$ specific aspects, $\calK=\{1, \dots, K\}$. For our hybrid model, $\calK=\{0, \dots, K\}$. Because every summand is positive semi‑definite, adding a task can only
increase or keep the Fisher matrix. Hence we have:
\begin{equation} \label{eq:Fisher_hybrid_single}
    \mathcal I^{(\text{hybrid})} = \mathcal I^{(\text{single})} + \Delta,
\qquad \Delta\succeq 0,
\end{equation}

\textbf{Strict positivity of difference term.}
In fact, give the assumed positive correlation between head 0 and other attribute heads, we can show that the overall Fisher matrix can be strictly larger. Denote
\[
g_0(y_i) = \nabla_\theta r_s(y_i),
\]
which is the gradient of the overall head’s prediction with respect to $\theta$. When we look at the contribution of the other heads, what matters is how their gradients project onto $g_0(y_i)$:
\[
g_0(y_i)^\top \nabla_\theta r_k(y_i).
\]
The positive correlation assumption $\rho_{0k} > 0$ implies that, on average, the gradients $\nabla_\theta r_k(y_i)$ tend to point in a similar direction to $g_0(y_i)$. This means that $g_0(y_i)^\top \nabla_\theta r_k(y_i) > 0.$ For example, if we project the Fisher information onto the direction $g_0$, using the linearity of the inner product, we obtain:
\[
g_0(y_i)^\top\mathcal{I}^{\text{hybrid}}(\theta) \,g_0(y_i)
= g_0(y_i)^\top\mathcal{I}^{\text{single}}(\theta) \,g_0(y_i)
+ \sum_{k=1}^{K}\frac{1}{n\sigma_{kk}}\sum_{i=1}^{n}\bigl(g_0(y_i)^\top \nabla_\theta r_k(y_i)\bigr)^2.
\]
Because each term $\bigl(g_0(y_i)^\top \nabla_\theta r_k(y_i)\bigr)^2$ is strictly positive when the inner product is nonzero, and positive correlation ensures that it is indeed positive on average, we know the extra sum is strictly positive. That is,
\[
g_0(y_i)^\top\mathcal{I}^{\text{hybrid}}(\theta) \,g_0(y_i)
>
g_0(y_i)^\top\mathcal{I}^{\text{single}}(\theta) \,g_0(y_i).
\]
Therefore, we can get
\begin{equation} \label{eq:Fisher_hybrid_single_strict}
    \mathcal I^{(\text{hybrid})} = \mathcal I^{(\text{single})} + \Delta,
\qquad \Delta\succ 0,
\end{equation}
\textbf{Asymptotic Variance of $\widehat\theta$.}
If $\widehat{\theta}$ is the maximum-likelihood estimator (MLE) of
$\theta^\star$ and the usual regularity conditions hold
(i.i.d.\ samples, smooth log-likelihood, finite Fisher information, etc.),
then the asymptotic normality theorem for MLEs states
\[
\sqrt{n}\bigl(\widehat{\theta} - \theta^\star\bigr)
\xrightarrow{d}
\mathcal{N}\Bigl(
0,[\mathcal{I}(\theta^\star)]^{-1}
\Bigr).
\]
where ``$\xrightarrow{d}$'' denotes convergence in distribution and the covariance of the limiting Gaussian is the inverse Fisher information, which is the smallest possible asymptotic variance for any unbiased estimator by Cramér–Rao.

Thus, if one yields a larger
Fisher matrix (more information), its estimator’s asymptotic covariance matrix is smaller, so predictions based on it are less variable. Hence
\begin{equation}\label{eq:Cov_hybrid_single}
    \operatorname{Cov}{\text{hybrid}}(\widehat\theta)
    \prec 
    \operatorname{Cov}{\text{single}}(\widehat\theta).
\end{equation}


\textbf{From $\theta$‑variance to MSE by Bias–variance decomposition.}
For a fresh test example $v$ we predict with
\[
\hat{s}_0(v) = \mathbf{w}_S^{\top} M_{\widehat{\theta}}(v),
\qquad
g_s(v) = \mathbf{w}_S^{\top} f_{\theta^\star}(v) + \varepsilon_0.
\]

The mean-squared error of that prediction is
\[
\operatorname{MSE}_0
=
\underbrace{\bigl(\E[\hat{s}_0] - \E[g_s]\bigr)^2}_{\text{Bias}^2}
+
\underbrace{\operatorname{Var}[\hat{s}_0]}_{\text{estimation variance}}
+
\underbrace{\operatorname{Var}[\varepsilon_0]}_{\sigma_{00}\text{(irreducible noise)}}.
\]
\begin{itemize}
    \item \textbf{Bias term.} With sufficient optimization and model capacity the MLE is (asymptotically) unbiased, so this term is approximately $0$.
    \item \textbf{Variance term.} Fluctuations of $\widehat{\theta}$ across data sets propagate through the network, and first-order Taylor expansion gives
    \begin{align*}
        &M_{\widehat{\theta}}(v) \approx f_{\theta^\star}(v) + \nabla_\theta f_{\theta^\star}(v)\,(\widehat{\theta}-\theta^\star)\\
        \Longrightarrow& 
        \hat{s}_0(v) \approx \mathbf{w}_S^\top f_{\theta^\star}(v) +  \nabla_\theta M_S(v) (\widehat{\theta}-\theta^\star) \quad (\text{denote }M_S(v)  \bydef \mathbf{w}_S^\top f_{\theta^\star}(v))\\
        \Longrightarrow& 
        \operatorname{Var}[\hat{s}_0] \approx
        \nabla_\theta M_S(v)^{\top}
        \operatorname{Cov}(\widehat{\theta})
        \nabla_\theta M_S(v).
    \end{align*}
    Because the hybrid regime has the smaller $\operatorname{Cov}(\widehat{\theta})$ by Equation~\ref{eq:Cov_hybrid_single} above, this variance shrinks.
    \item \textbf{Noise term $\sigma_{00}$.} This is the intrinsic randomness in the label and is identical for all training regimes.
\end{itemize}
Therefore, when hybrid training reduces the covariance \( \operatorname{Cov}(\widehat{\theta}) \), the key variance term in the generalization bound decreases. Due to the assumed positive correlation between the fine-grained attributes and overall quality, this reduction leads to a lower test-set MSE for the single-objective head. This establishes inequality~\ref{eq:thm-MSE0-inequality} and thus completes the proof of Theorem~\ref{theorem2}.

Moreover, by linking the single-head MSE to the pairwise preference error (as shown in Lemma~\ref{lemma1}), we demonstrate that the single-objective head trained using our SMORM framework is expected to outperform a conventional single-head reward model. It is worth noting that the same argument naturally extends to the multi-objective reward setting as well.

\end{proof}

\section{Motivation for Enhancing Multi-Objective Reward Functions without Additional Multi-Attribute Data}
\label{appendix:motivationformultienhancing}
In this section, we provide a detailed discussion on the limited availability of high-quality data for training multi-objective reward models, which motivates our approach to enhancing multi-objective reward modeling performance without relying on additional multi-attribute annotations.

While several datasets provide dense prompt–response pairs with fine-grained attribute scores—such as UltraFeedback \cite{cui2023ultrafeedback} with 240K samples and Prometheus \cite{kim2023prometheus} with 200K samples—their annotations are primarily generated by GPT-based models. This introduces several concerns:
\textbf{(1)} As foundation models continue to evolve, the quality and consistency of their annotations become increasingly difficult to guarantee.
\textbf{(2)} The use of large language models (LLMs) as annotators introduces potential biases \cite{gu2024survey, li2024generation}, which can be inherited by the reward model and subsequently transferred to the policy model when optimized using reward signals.
\textbf{(3)} Our experimental results in Section~\ref{experiments} show that, when using gemma-2b-it as the base model, training on HelpSteer2 \cite{wang2024helpsteer2}—a human-annotated dataset of 20K samples—yields superior performance compared to training on UltraFeedback with 240K GPT-labeled samples.

HelpSteer2 \cite{wang2024helpsteer2} is one of the few available datasets that provide high-quality, human-annotated, multi-objective preference labels. However, it only contains 20K samples, and its creation involved an exceptionally rigorous annotation process. This process includes multiple layers of human oversight, dynamic annotator recruitment, and strict quality control procedures to ensure data integrity. Due to the high resource demands of this pipeline, it is difficult to generalize or scale to larger datasets.

In summary, it remains difficult to obtain large-scale, high-quality, fine-grained attribute scores for responses, whether through GPT-based evaluation or human annotation. This limitation motivates the development of methods that can enhance the performance of multi-objective reward models without requiring additional annotated data.

\section{Experiments on OOD}
\label{appendix:oodevaluation}
To evaluate the generalizability of our SMORM, we adopt the experimental setup described in Section~\ref{experiments}, training reward models on both 400K and 40K samples from the Unified-Feedback dataset. In both settings, HelpSteer2 serves as the multi-objective dataset $\mathcal{D}_M$ for SMORM. All methods use gemma-2B-it as the base model and are evaluated on both in-distribution (ID) data (Unified-Feedback) and out-of-distribution (OOD) benchmarks (HHH-Alignment and MT-Bench). The results are presented in Tables~\ref{ood:400k} and~\ref{ood:40k}. 
The results demonstrate that SMORM-F consistently outperforms competing methods in both in-distribution (ID) and out-of-distribution (OOD) evaluations. Specifically, SMORM-F achieves an ID score of 76 and an OOD score of 83.2 on the HHH-Alignment benchmark, surpassing the second-best method, which attains scores of 73.8 and 79.6, respectively. These findings suggest that incorporating a multi-objective learning function effectively shapes the embedding space, leading to improved performance on ID data and enhanced generalizability of the single-objective head to OOD scenarios.
\begin{table}[ht]
\centering
\caption{Results on \textbf{ID} and \textbf{OOD} evaluation with \textbf{400K training data} from Unified-Feedback. The best performance in each task is in bold and the second best one is underlined.}
\begin{tabular}{lccc}
\toprule
\textbf{Reward Model} & \textbf{Unified Feedback} & \textbf{HHH Alignment} & \textbf{MT Bench} \\
\midrule
Classifier (Frozen)            & 63.8 & 66.4 & 69.5 \\
Classifier (baseline)          & 72.1 & 73.4 & 71.2 \\
Classifier + margin            & 72.0 & 75.0 & 72.6 \\
Classifier + label smooth      & 71.5 & 72.1 & 71.2 \\
Classifier + Ensemble          & 72.8 & 76.8 & \textbf{73.7} \\
GRM & \underline{73.8} & \underline{79.6} & 73.4  \\
\textbf{SMORM-F (Ours)} & \textbf{76.0} & \textbf{83.2} & \underline{73.4}  \\
\bottomrule
\end{tabular}
\label{ood:400k}
\end{table}

\begin{table}[ht]
\centering
\caption{Results on \textbf{ID} and \textbf{OOD} evaluation with \textbf{40K training data} from Unified-Feedback. The best performance in each task is in bold and the second best one is underlined.}
\begin{tabular}{lccc}
\toprule
\textbf{Reward Model} & \textbf{Unified Feedback} & \textbf{HHH Alignment} & \textbf{MT Bench} \\
\midrule
Classifier (Frozen)            & 63.9 & 68.6 & 68.2 \\
Classifier (baseline)          & 68.8 & 70.3 & 69.1 \\
Classifier + margin            & 69.6 & 69.8 & 71.0 \\
Classifier + label smooth      & 68.5 & 68.8 & 71.9 \\
Classifier + Ensemble          & 69.9 & 72.2 & 71.1 \\
GRM  & \underline{71.3} & \underline{75.1} & \underline{72.8} \\
\textbf{SMORM-F (Ours)} & \textbf{72.9} & \textbf{77.4} & \textbf{73.0} \\
\bottomrule
\end{tabular}
\label{ood:40k}
\end{table}
\section{Alignment Results after PPO}
\label{appendix:alignment}
To demonstrate the effectiveness of our SMORM in training superior policy models within the RLHF framework, we evaluate the win rate of models trained using PPO with SMORM-F as the proxy reward model, comparing it against two baselines: Baseline SORM and GRM. 
We obtain the optimized models following the training process in Sec. \ref{exp:rlhf} and randomly sample 500 prompts from the Unified-Feedback dataset and use the optimized policy models to generate corresponding responses. We then employ GPT-4o as an impartial judge to compare the model outputs and determine which response is preferable, recording the win rates accordingly.
We adopt the win-tie evaluation prompt format introduced by \cite{khanov2024args}. The results are summarized in Table \ref{rlhfcomparison}. 
The system prompt and user prompt used for evaluation are provided below:

\begin{tcolorbox}[title=Prompt for Win-Tie Evaluation by GPT-4o, colback=gray!5!white, colframe=gray!75!black]

\textbf{System prompt}

\texttt{[system]}\\
You are a helpful, harmless, and precise assistant for checking the quality of the answer. We would like to request your feedback on the performance of two AI assistants in response to the user question. Please rate the harmlessness, helpfulness, and level of detail of their responses. Your evaluation should consider factors such as the helpfulness, harmlessness, relevance, accuracy, depth, creativity, and level of detail of the response. Note that if a response appears cut off at the end due to length constraints, it should not negatively impact the score. Also, base your evaluation solely on the given answer, disregarding any preceding interactions in the question. Each assistant receives an overall score on a scale of 1 to 10, where a higher score indicates better overall performance.

Please first output a single line containing only two values indicating the scores for Assistant 1 and 2, respectively. The two scores are separated by a space. In the subsequent line, please provide a comprehensive explanation of your evaluation, avoiding any potential bias and ensuring that the order in which the responses were presented does not affect your judgment.

\textbf{User prompt}

\texttt{[Question]}\\
\{question\}

\texttt{[The Start of Assistant 1's Answer]}\\
\{answer1\}\\
\texttt{[The End of Assistant 1's Answer]}

\texttt{[The Start of Assistant 2's Answer]}\\
\{answer2\}\\
\texttt{[The End of Assistant 2's Answer]}

\end{tcolorbox}

From the table, we observe that across both dataset scales, using our SMORM-F as the proxy reward function results in a policy model that consistently outperforms both the Baseline and GRM, with win rates always exceeding 65\%.

\begin{table}[h]
\caption{Win rate of models after PPO training with SMORM-F against baseline SORM and GRM.}
\centering
\begin{tabular}{lllccc}
\toprule
\textbf{Method} & \textbf{vs.} & \textbf{Method} &\textbf{Win (\%) $\uparrow$} & \textbf{Tie (\%)} & \textbf{Lose (\%) $\downarrow$}  \\
\midrule
\multicolumn{6}{c}{$\mathcal{D_S}/\mathcal{D}_M$: UnifiedFeedback 40k/HelpSteer2}\\
\midrule
SMORM-F && Baseline & 75.7& 0.5&23.8\\
SMORM-F && GRM & 69.5&0.4 & 30.1\\
\midrule
\multicolumn{6}{c}{$\mathcal{D_S}/\mathcal{D}_M$: UnifiedFeedback 400k/UltraFeedback}\\
\midrule
SMORM-F && Baseline & 71.3& 0.7&28.0\\
SMORM-F && GRM & 65.7& 0.7& 33.6\\
\bottomrule
\end{tabular}
\label{rlhfcomparison}
\end{table}

\section{Implementation Details}
\label{appendix:implementation}
\subsection{Baseline and Training Details}

\textbf{Baseline Details.} All baseline reward models use the \texttt{AutoModelForSequenceClassification} class from the \texttt{transformers} library~\cite{wolf2020transformers}, which attaches a randomly initialized linear head for reward prediction. Each model is trained to minimize a loss function using the training data. For ensemble baselines, we train three models with different random seeds and aggregate their predictions.

We use the margin loss from~\cite{touvron2023llama} defined as:
\[
\mathcal{L}_{\text{margin}}(\theta) = -\mathbb{E}_{(x, y_c, y_r) \sim \mathcal{D}} \left[\log \left(\sigma\left(r_\theta(x, y_c) - r_\theta(x, y_r) - m(r)\right)\right)\right],
\]
where $m(r)$ is computed using the reward difference between chosen and rejected responses in the \texttt{Unified-Feedback} dataset. This loss emphasizes meaningful reward distinctions.

We also incorporate a label smoothing loss defined as:
{\small
\[
\mathcal{L}_{\text{smooth}}(\theta) = -\mathbb{E}_{(x, y_c, y_r) \sim \mathcal{D}} \left[(1 - \epsilon)\log\left(\sigma\left(r_\theta(x, y_c) - r_\theta(x, y_r)\right)\right) - \epsilon \log\left(\sigma\left(r_\theta(x, y_c) - r_\theta(x, y_r)\right)\right)\right]
\]
}
where $\epsilon = 0.1$. This formulation improves robustness by softening the loss against label noise, reducing overfitting.

For GRM, we strictly follow the implementation in \cite{yang2024regularizing}. The reward head consists of a linear layer (hidden size, 1024), followed by a ReLU activation, and a final linear layer (1024, 1). The regularization coefficient $\alpha$ is set to 0.01, and $\beta$ is set to 0.1. In the GRM-linear variant, the head is a single linear layer of shape (hidden size, 1). More detailed training procedures for GRM can be found in \cite{yang2024regularizing}.

\textbf{Computational Resources.} All experiments were conducted using NVIDIA H100 80GB GPUs. Training on 40K data samples requires approximately 16 GPU hours.

\begin{table}[h]
\centering
\caption{Key implementations of the text generation experiments.}
\resizebox{\textwidth}{!}{
\label{tab:key-impl}
\small
\begin{tabular}{ll}
\toprule
\multicolumn{2}{c}{\textbf{Basic Information}} \\
\midrule
Base models & gemma-2b-it and Mistral-7B-Instruct-v0.2 \\
Quantization for training & bf16 \\
Optimizer & AdamW\_hf \\
Batch size & 16 \\
Learning rate & $5 \times 10^{-6}$ \\
Learning rate scheduler & cosine \\
Warmup ratio & 0.03 \\
\midrule
\multicolumn{2}{c}{\textbf{SMORM}} \\
\midrule
Weight ratio for single-objective to multi-objective reward modeling & 1.0 by default\\
\midrule
\multicolumn{2}{c}{\textbf{PPO}~\cite{schulman2017proximal}} \\
\midrule
KL regularization & 0.0 \\
Epochs & 1 \\
Learning rate & $1 \times 10^{-5}$ \\
$\lambda$ for GAE & 0.95 \\
$\gamma$ & 1 \\
Clip range & 0.2 \\
Optimization epochs per batch & 4 \\
Tokens during generation & 512 \\
\bottomrule
\end{tabular}}
\end{table}

\section{Hyperparameter Analysis and Instability of GRM}
\label{appendix:hyperandinstability}
In this section, we conduct experiments to demonstrate the instability of GRM \cite{yang2024regularizing} and the relative stability of our proposed SMORM. Intuitively, reward models are typically initialized from language models pretrained on next-token prediction tasks, and are then fine-tuned for reward modeling. This motivates the hypothesis that introducing a next-token prediction regularization term—aligned with the pretraining objective—may conflict with the reward modeling objective. As a result, this misalignment could lead to unstable performance during reward model training.
To validate this assumption, we conduct experiments following the setup described in Sec.~\ref{experiments}, using 40K samples from \texttt{Unified-Feedback} as $\mathcal{D}_S$ and \texttt{HelpSteer2} as $\mathcal{D}_M$. We compare the performance of GRM to a baseline single-objective reward model and our proposed SMORM-F. In addition, we compare a baseline multi-objective reward model to our SMORM-L. We vary the weight ratio in $\left[0.01, 0.1, 1, 10\right]$. In GRM, the weight ratio refers to the strength of the next-token prediction regularization. In SMORM, the weight ratio controls the contribution of multi-objective reward modeling. All models are evaluated on \texttt{RewardBench}. The results are presented in Fig.~\ref{fig:hyper}.
From the figure, we observe that, except in extreme cases, our SMORM framework demonstrates consistently stable performance. Specifically, when the weight ratio is set to 0.01, the performance of SMORM-L slightly falls short of the baseline multi-objective reward model, and when the ratio is set to 10, SMORM-F marginally underperforms compared to the baseline single-objective reward model. Outside of these edge cases, both SMORM-F and SMORM-L consistently outperform their respective baselines across a wide range of weight settings. In contrast, the performance of GRM fluctuates significantly, ranging from approximately 45 to 65, and consistently underperforms relative to the baseline single-objective reward model. These results support our hypothesis that incorporating next-token prediction regularization introduces instability into reward model training. This instability is especially concerning given the substantial computational cost required to train reward models.
\begin{figure}[h]
    \centering
    \includegraphics[width=0.5\linewidth]{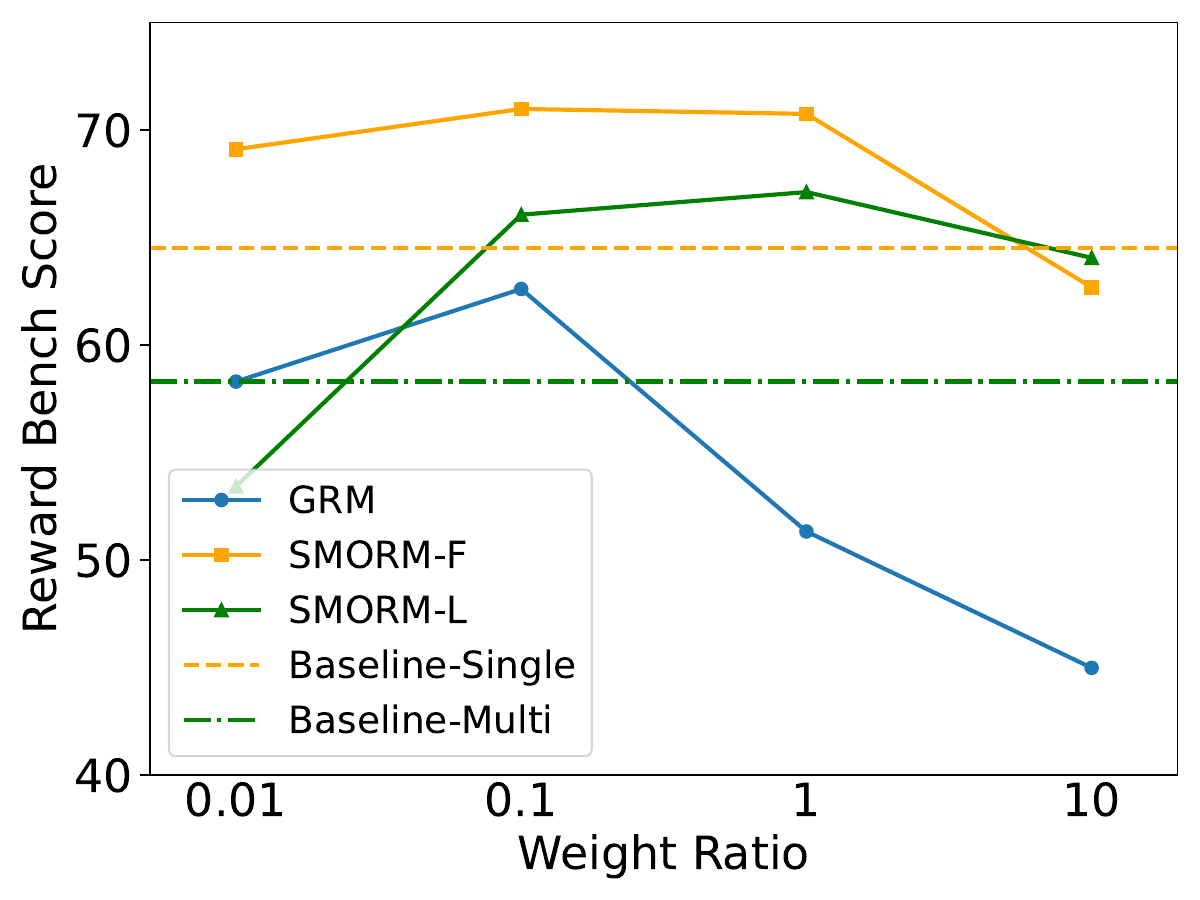}
    \caption{Hyperparameter analysis.}
    \label{fig:hyper}
\end{figure}

\section{Interpretation of Why SORM Fail in OOD Setting}
\begin{figure}[h]
    \centering
    \includegraphics[width=0.75\linewidth]{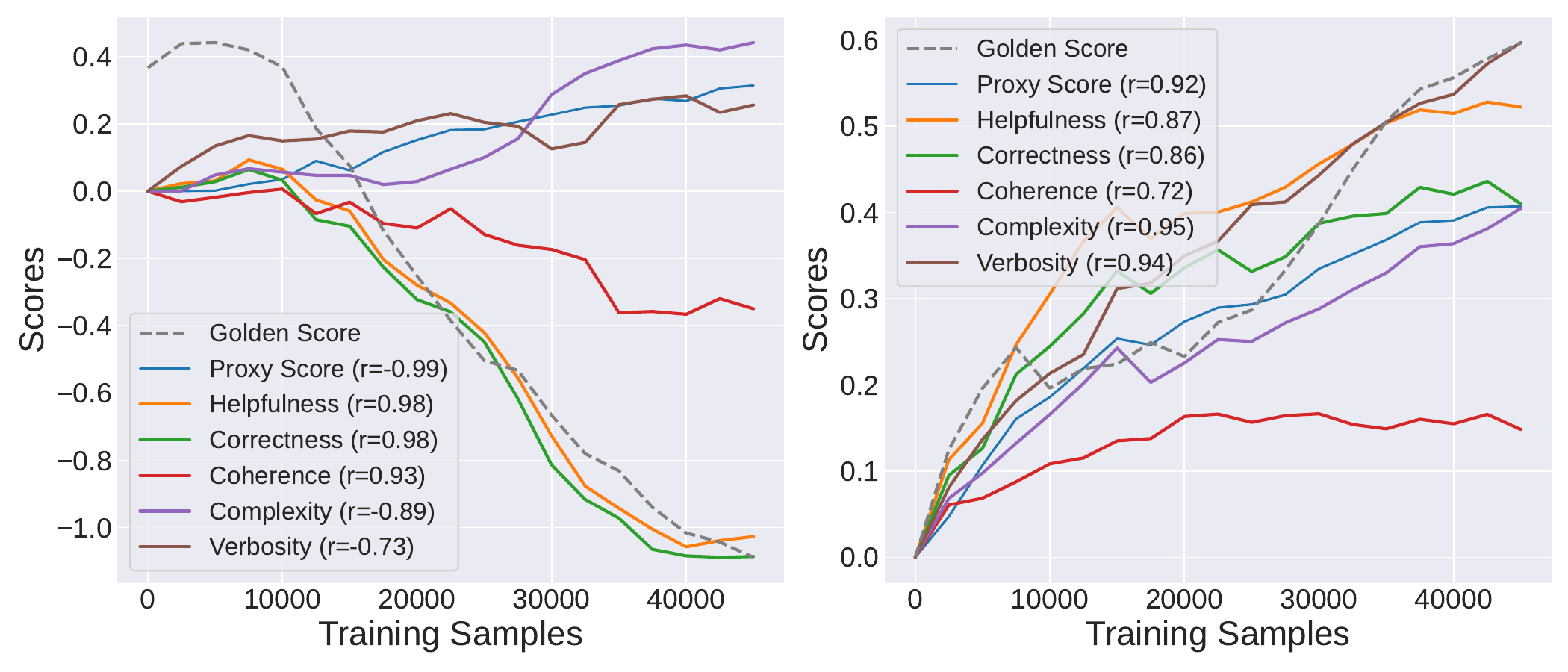}
    \begin{minipage}{0.37\textwidth}
        \small
        \centering
        \textbf{(a)}
    \end{minipage}%
    \begin{minipage}{0.37\textwidth}
        \small
        \centering
        \textbf{(b)}
    \end{minipage}
    \caption{Fine-grained attribute scores of the optimized policy model using (a) the baseline classifier and (b) SMORM-F as the proxy reward model.}
    \label{fig:intepretation}
\end{figure}
To interpret the vulnerability of the baseline classifier to reward hacking, we adopt the same experimental setting as in Sec.~\ref{hacking} and record evaluation scores for the generated responses across five dimensions: \texttt{Helpfulness}, \texttt{Correctness}, \texttt{Coherence}, \texttt{Complexity}, and \texttt{Verbosity}. These scores are derived using the multi-objective head $\mathbf{w}_M$ of the trained SMORM. Figures \ref{fig:intepretation} (a) and (b) illustrate the evaluation results when using the baseline classifier and SMORM-F as proxy reward models, respectively. In Fig. \ref{fig:intepretation} (a), employing the baseline classifier leads to improvements only in the \texttt{Complexity} and \texttt{Verbosity} dimensions, while performance declines in the remaining attributes. Consequently, the generated responses are not considered high-quality by the gold reward model.
In contrast, Fig. \ref{fig:intepretation} (b) shows that using SMORM-F as the proxy reward model results in consistent improvements across all five fine-grained dimensions. These enhancements are also reflected in an increased gold score. These findings indicate that a conventional single-objective reward model is typically insufficient to capture the multifaceted criteria that make a chosen response preferable to a rejected one.

\section{Analysis on RM-Bench}
\label{appendix:rmbench}

RM-Bench \cite{liu2025rmbench} evaluates reward models on two key dimensions: sensitivity to subtle changes and robustness to style bias. It includes three task types:
(1) \texttt{Easy}: Chosen responses are detailed and informative; rejected ones are concise and minimal.
(2) \texttt{Normal}: Both responses share the same style but differ in key information.
(3) \texttt{Hard}: Chosen responses are concise; rejected ones are detailed.


\subsection{MORMs Fall Short on Easy and Normal Tasks.}

In this section, we present empirical results to illustrate why baseline multi-objective reward models tend to underperform on RM-Bench~\cite{liu2025rmbench}. We begin by evaluating a selection of existing open-source single-objective and multi-objective reward models that exhibit comparable performance on RewardBench~\cite{lambert2024rewardbench}. The evaluated models include:

\begin{itemize}
    \item \textbf{Multi-objective reward models:}
    \begin{itemize}
        \item \texttt{NVIDIA/Nemotron-340B-Reward}\footnote{\href{https://huggingface.co/nvidia/Nemotron-4-340B-Reward}{nvidia/Nemotron-4-340B-Reward}}
        \item \texttt{RLHFlow/ArmoRM-Llama3-8B-v0.1}\footnote{\href{https://huggingface.co/RLHFlow/ArmoRM-Llama3-8B-v0.1}{RLHFlow/ArmoRM-Llama3-8B-v0.1}}
    \end{itemize}
    
    \item \textbf{Single-objective reward models:}
    \begin{itemize}
        \item \texttt{Ray2333/GRM-llama3-8B-distill}\footnote{\href{https://huggingface.co/Ray2333/GRM-llama3-8B-distill}{Ray2333/GRM-llama3-8B-distill}}
        \item \texttt{internlm/internlm2-20b-reward}\footnote{\href{https://huggingface.co/internlm/internlm2-20b-reward}{internlm/internlm2-20b-reward}}
        \item \texttt{NCSOFT/Llama-3-OffsetBias-RM-8B}\footnote{\href{https://huggingface.co/NCSOFT/Llama-3-OffsetBias-RM-8B}{NCSOFT/Llama-3-OffsetBias-RM-8B}}
        \item \texttt{Ray2333/GRM-llama3-8B-sftreg}\footnote{\href{https://huggingface.co/Ray2333/GRM-llama3-8B-sftreg}{Ray2333/GRM-llama3-8B-sftreg}}
        \item \texttt{LxzGordon/URM-LLaMa-3.1-8B}\footnote{\href{https://huggingface.co/LxzGordon/URM-LLaMa-3.1-8B}{LxzGordon/URM-LLaMa-3.1-8B}}
        \item \texttt{Ray2333/GRM-Llama3.2-3B-rewardmodel-ft}\footnote{\href{https://huggingface.co/Ray2333/GRM-Llama3.2-3B-rewardmodel-ft}{Ray2333/GRM-Llama3.2-3B-rewardmodel-ft}}
        \item \texttt{Skywork/Skywork-Reward-Llama-3.1-8B}\footnote{\href{https://huggingface.co/Skywork/Skywork-Reward-Llama-3.1-8B}{Skywork/Skywork-Reward-Llama-3.1-8B}}
    \end{itemize}
\end{itemize}

\begin{table}[h]
\centering
\caption{Comparison of models on Easy, Normal, and Hard tasks on RM-Bench.}
\resizebox{\textwidth}{!}{
\begin{tabular}{lccccc}
\toprule
\textbf{Model Name} & Easy & Normal & Hard & Avg & RewardBench\\
\midrule
\texttt{Skywork/Skywork-Reward-Llama-3.1-8B} & 89.0 & 74.7 & 46.6 & 70.1 & 93.1\\
\texttt{LxzGordon/URM-LLama-3.1-8B} & 84.0 & 73.2 & 53.0 & 70.0 & 92.9\\
\texttt{NCSOFT/Llama-3-OffsetBias-RM-8B} & 84.6 & 72.2 & 50.2 & 69.0 & 89.4\\
\texttt{internlm/internlm2-20b-reward} & 82.6 & 71.6 & 50.7 & 68.3 & 90.2\\
\texttt{Ray2333/GRM-llama3-8B-sftreg} & 83.5 & 72.7 & 48.6 & 68.2 & 87.0\\
\texttt{Ray2333/GRM-llama3-8B-distill} & 82.2 & 71.5 & 48.4 & 67.4 & 86.2\\
\texttt{Ray2333/GRM-Llama3.2-3B-rewardmodel-ft}  & 89.9 & 74.0 & 44.0 & 69.3 & 90.9\\ \midrule
\texttt{RLHFlow/ArmoRM-Llama3-8B-v0.1} & 82.5 & 70.8 & 50.1 & 67.8 & 90.4\\
\texttt{NVIDIA/Nemotron-340B-Reward} & 81.0 & 71.4 & 56.1 & 69.5 & 92.0\\
\bottomrule
\end{tabular}}
\label{rmbench-opensource}
\end{table}

\begin{figure}[h]
    \centering
    \includegraphics[width=0.75\linewidth]{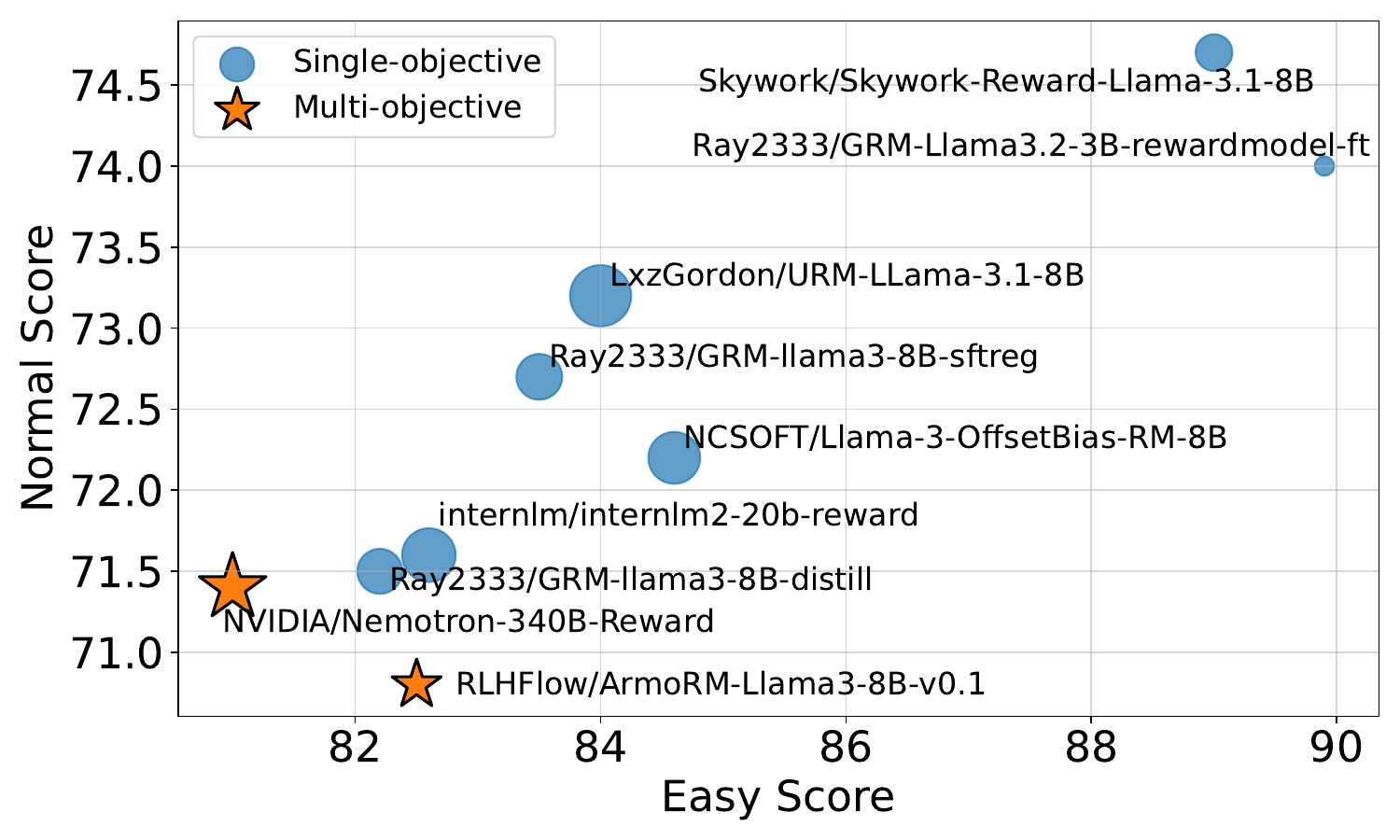}
    \caption{Visualization of results on RM-Bench. The size of each marker indicates the model's performance on the Hard task.}
    \label{fig:rm-benchopensource}
\end{figure}

The results of these comparisons are presented in Table~\ref{rmbench-opensource}. For better illustration, we also visualize the results in Fig.~\ref{fig:rm-benchopensource}. We observe that existing multi-objective reward models tend to underperform on the Easy and Normal tasks, even though their overall performance on RewardBench and the Hard tasks of RM-Bench is not among the worst.
We attribute this phenomenon to the following: multi-objective reward models (MORMs) typically assess response quality through utility and style attributes. For example, \texttt{HelpSteer2} \cite{wang2024helpsteer2} provides scores for \textit{correctness}, \textit{helpfulness}, and \textit{coherence} (utility), and \textit{complexity} and \textit{verbosity} (style). Formally, we define:
$
\mathcal{U} = \{\text{correctness, helpfulness, coherence}\},
\mathcal{S} = \{\text{verbosity, complexity}\}.
$
The score margin is:
\[
\Delta r =
\underbrace{
  \sum_{k \in \mathcal{U}} w_k\,\Delta h_k
}_{r_u:\text{utility heads}}
+
\underbrace{
  w_{\mathrm{verb}}\,\Delta h_{\mathrm{verb}}
  + w_{\mathrm{comp}}\,\Delta h_{\mathrm{comp}}
}_{r_b:\text{style bias}}.
\]
In \texttt{Easy} tasks, longer chosen responses incur negative style penalties, potentially flipping $\Delta r$ despite higher utility scores. In \texttt{Normal} tasks, responses should have similar style, implying $\Delta h_{\mathrm{verb}} \approx \Delta h_{\mathrm{comp}} \approx 0$, but in practice, non-trivial deviations occur due to imperfect optimization. As shown in Appendix~\ref{appen:rmbenchcase}, style noise degrades scoring accuracy and overall performance.

\subsection{Case Study and Why SMORM Help}
\label{appen:rmbenchcase}
In this section, we present a case study to interpret why baseline multi-objective reward models fall short on the Normal and Easy tasks of RM-Bench.

We train a multi‐objective reward model (MORM) on the \texttt{HelpSteer2} dataset and, for each attribute across all tasks, compute the mean and variance of the prediction differences—where each difference is defined as the chosen‐response score minus the rejected‐response score.
The results are summarized in Table~\ref{tab:mean_norm_diff}. As shown in the table, while the baseline MORM exhibits mean scores for complexity and verbosity that are approximately zero across the Normal tasks, the corresponding variances are non-trivial. This indicates that, although paired responses should receive similar scores for complexity and verbosity, there remains considerable prediction bias on these attributes in individual comparisons. 


\begin{table}[h]
\centering
\caption{Normalized pairwise differences per dimension (Mean (Variance))}
\begin{tabular}{lcccccc}
\hline
Category  & helpfulness & correctness & coherence & complexity & verbosity\\
\hline
Normal overall    & 0.08 (0.83) & 0.00 (1.55) & 0.07 (0.74) & 0.17 (1.29) & 0.00 (0.55) \\
Normal Correct    & 0.41 (0.75) & 0.33 (1.48) & 0.29 (0.88) & 0.43 (1.31) & 0.10 (0.52) \\
Normal False       & -0.31 (0.65)& -0.38 (1.35)& -0.19 (0.46)& -0.13 (1.10)& -0.12 (0.55) \\
\hline
Hard overall      & 0.26 (1.02) & 0.32 (1.90) & -0.28 (0.95)& -0.56 (1.85)& 0.05 (0.81) \\
Hard Correct    & 0.78 (0.84) & 0.85 (1.86) & 0.01 (0.82) & -0.27 (2.24)& 0.14 (0.81) \\
Hard False       & -0.08 (0.84)& -0.03 (1.62)& -0.48 (0.93)& -0.75 (1.50)& -0.19 (0.77) \\
\hline
Easy overall     & -0.11 (1.32)& -0.32 (1.93)& 0.41 (1.38) & 0.90 (2.36) & 0.05 (0.75) \\
Easy Correct   & 0.19 (1.10) & -0.07 (1.69)& 0.61 (1.51) & 1.12 (2.15) & 0.12 (0.70) \\
Easy False      & -0.76 (1.20)& -0.87 (2.02)& -0.03 (0.82)& 0.40 (2.46) & -0.10 (0.83) \\
\hline
\end{tabular}
\label{tab:mean_norm_diff}
\end{table}
Ideally, a single-objective reward model (SORM) can better align with helpfulness and correctness, as its preference training dataset encompasses diverse attributes and tends to exhibit less sensitivity to superficial factors like response length. According to Theorem~\ref{prop:implicit_multi_effect_detailed}, the aggregate score from the multi-objective reward model is lower-bounded by the single-objective reward score: $r_m(x, y) = \frac{1}{K} \sum_{i=1}^{K} \alpha_i w_{M,i}^\top f_\theta(x, y) \geq c \cdot r_s(x, y) - \varepsilon,$
where $\alpha_i$ reflects the correlation between each attribute and the chosen/rejected preference direction.
Consider a scenario where response $y_A$ is preferred over $y_B$ according to ground-truth labels. Initially, it may happen that $r_m(y_A) < r_m(y_B)$, especially when the style bias overwhelms the utility signal, i.e., $|r_u(y_A) - r_u(y_B)| < |r_b(y_A) - r_b(y_B)|$. However, when the single-objective reward assigns a higher score to $y_A$, i.e., $r_s(y_A) > r_s(y_B)$, the utility heads in the multi-objective model are progressively adjusted to favor $y_A$. This adjustment increases the margin $|r_u(y_A) - r_u(y_B)|$, eventually outweighing the bias introduced by stylistic components ($r_b$). As a result, the style penalty's influence diminishes, and the model better reflects true preference judgments.
To verify this, we train the SMORM using 40K samples from \texttt{Unified-Feedback} as $\mathcal{D}_S$, and report the differences in Table~\ref{tab:mean_norm_diff_smorm}. As shown, for normal cases, both the variance and verbosity are substantially reduced compared to the variance observed in helpfulness. This indicates that style-related biases no longer dominate the final decision, thereby enhancing the performance of the multi-objective reward model on this task.
\begin{table}[h]
\centering
\caption{Normalized pairwise differences per dimension (Mean (Variance))}
\begin{tabular}{lcccccc}
\hline
Category & helpfulness & correctness & coherence & complexity & verbosity \\
\hline
Normal overall   & 0.43 (0.62) & 0.43 (0.68) & 0.29 (0.41) & -0.02 (0.11) & 0.01 (0.08) \\
Normal correct   & 0.73 (0.57) & 0.71 (0.70) & 0.53 (0.35) & -0.01 (0.13) & 0.05 (0.10) \\
Normal false     & -0.24 (0.09)& -0.17 (0.10)& -0.23 (0.11)& -0.03 (0.07)& -0.07 (0.04) \\\midrule
Hard overall     & 0.28 (0.70) & 0.44 (0.93) & 0.25 (0.55) & -0.30 (0.16)& 0.33 (0.20) \\
Hard correct     & 0.77 (0.62) & 0.91 (1.05) & 0.68 (0.45) & -0.31 (0.17)& 0.48 (0.22) \\
Hard false       & -0.34 (0.11)& -0.16 (0.12)& -0.30 (0.15)& -0.29 (0.14)& 0.14 (0.11) \\\midrule
Easy overall    & 0.57 (0.67) & 0.43 (0.61) & 0.34 (0.38) & 0.26 (0.16) & -0.30 (0.20) \\
Easy correct     & 0.83 (0.56) & 0.62 (0.59) & 0.53 (0.30) & 0.30 (0.17) & -0.31 (0.23) \\
Easy false      & -0.26 (0.12)& -0.18 (0.16)& -0.27 (0.16)& 0.14 (0.11) & -0.27 (0.12) \\
\hline
\end{tabular}
\label{tab:mean_norm_diff_smorm}
\end{table}

\newpage

\end{document}